\documentclass{article}

\usepackage[preprint]{neurips_2024}

\usepackage[utf8]{inputenc} 
\usepackage[T1]{fontenc}    
\usepackage{hyperref}       
\usepackage{url}            
\usepackage{booktabs}       
\usepackage{amsfonts}       
\usepackage{nicefrac}       
\usepackage{microtype}      
\usepackage{xcolor}         
\usepackage{graphicx}
\usepackage{wrapfig}
\usepackage{listings}
\usepackage{caption}
\usepackage{subcaption}
\usepackage{ulem}
\usepackage{amsmath}
\usepackage{bbding}
\usepackage{amsthm}  
\newtheorem{theorem}{Theorem}  

\usepackage{siunitx}
\usepackage{multirow} 

\usepackage[capitalize]{cleveref} 
\usepackage[linesnumbered,ruled,vlined]{algorithm2e}
\graphicspath{{Figs/}}

\newcommand{\methodname}{Coca-Splat}
\newcommand{\camdfa}{CaMDFA}
\newcommand{\rayref}{RefRay}
\title{Coca-Splat: Collaborative Optimization for Camera Parameters and 3D Gaussians}

\author{%
Jiamin Wu$^{1,2}$\thanks{This work is done during an internship in the International Digital Economy Academy (IDEA).}~~, Hongyang Li $^{2}$~~, Xiaoke Jiang$^2$\thanks{Corresponding author}~~, Yuan YAO$^1$~~, Lei Zhang$^2$\\
$^1$Hong Kong University of Science and Technology \\
$^2$International Digital Economy Academy (IDEA)\\
} 

\begin{document}

\maketitle

\begin{abstract}
In this work, we introduce \textbf{{\methodname}}, a novel approach to addressing the challenges of sparse view pose-free scene reconstruction and novel view synthesis (NVS) by jointly optimizing camera parameters with 3D Gaussians.
Inspired by deformable DEtection TRansformer, we design separate queries for 3D Gaussians and camera parameters and update them layer by layer through deformable Transformer layers, enabling joint optimization in a single network. This design demonstrates better performance because to accurately render views that closely approximate ground-truth images relies on precise estimation of both 3D Gaussians and camera parameters. In such a design, the centers of 3D Gaussians are projected onto each view by camera parameters to get projected points, which are regarded as 2D reference points in deformable cross-attention. With camera-aware multi-view deformable cross-attention (\textbf{{\camdfa}}), 3D Gaussians and camera parameters are intrinsically connected by sharing the 2D reference points. Additionally, 2D reference point determined rays (\textbf{RayRef}) defined from camera centers to the reference points assist in modeling relationship between 3D Gaussians and camera parameters through RQ-decomposition on an overdetermined system of equations derived from the rays, enhancing the relationship between 3D Gaussians and camera parameters.
Extensive evaluation shows that our approach outperforms previous methods, both pose-required and pose-free, on RealEstate10K and ACID within the same pose-free setting.

\end{abstract}

\section{Introduction} \label{sec:intro}
Sparse view scene reconstruction, novel view synthesis (NVS), and input view pose estimation represent crucial challenges in computer vision with wide-ranging applications in robotics, augmented reality (AR), virtual reality (VR), games, and films. Recent advancements in differentiable rendering techniques such as Neural Radiance Fields (NeRF) \citep{nerf, nerf++} and 3D Gaussian Splatting (3DGS) \citep{3d-gs, 3d_gs_survey} have enabled rapid progress in scene reconstruction and NVS using deep learning methods, relying solely on 2D supervision due to difficulty in obtaining 3D assets for supervision. Recent methods like pixelSplat \citep{pixelsplat} and MVSplat \citep{mvsplat} take advantage of fast rendering speed of 3DGS to achieve precise scene reconstruction and NVS within approximately one second. However, these methods require ground truth camera poses for input views, posing challenges in real-world scenarios.
While these poses can be derived from dense videos using structure-from-motion (SfM) methods \citep{colmap, sfm1, sfm2}, inaccurate pose estimation may impact performance and increase inference time.
\begin{figure}[t]
    \centering
    \includegraphics[width=0.77\textwidth]{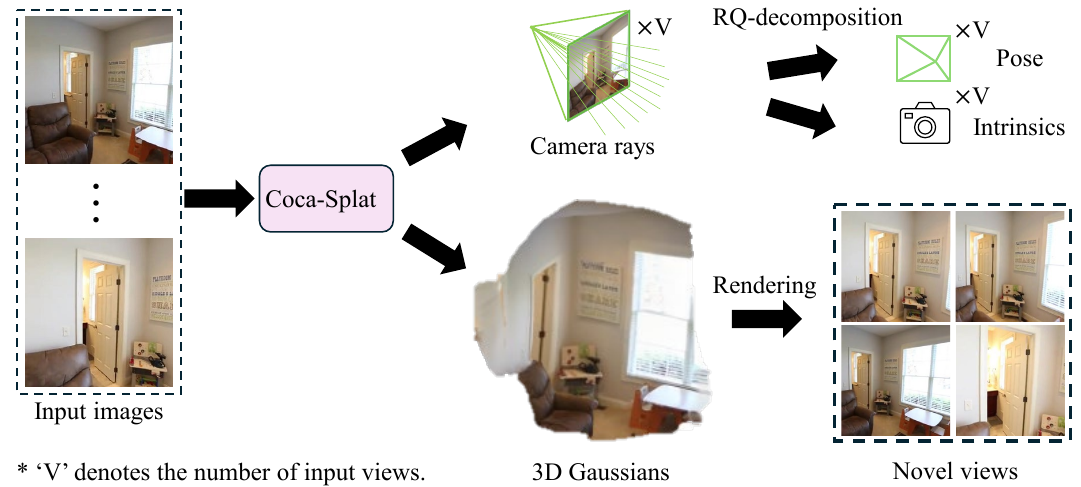}
    
    \caption{\textbf{{\methodname}}. Given sparse unposed images, our method reconstructs 3D Gaussians and camera rays using a feed-forward network. Subsequently, camera parameters are derived from camera rays, and novel views are rendered from 3D Gaussians.}
    
    \label{fig: teaser}
\end{figure}

Recent methods focusing on pose-free scene reconstruction are increasingly suitable for real-world applications, falling into two main categories.
The first category involves pose-free scene reconstruction and NVS methods like \cite{noposplat, sparp, dust3r, mast3r, splatt3r, pf-lrm, gsloc} which do not explicitly incorporate pose information in the network architecture. Instead, they generate scenes based on image feature matching and comparison while utilizing post-processing techniques, such as Perspective-n-Point (PnP) algorithm \citep{sfm1}, for calculating input poses. The second category includes methods like \cite{forge, nope-nerf, selfsplat, dbarf, flowcam, unifying} which employ separate pose estimators for pose estimation and 3D estimators for scene reconstruction but obtain pose and scene information in a unified pipeline. Some of these methods combine predicted pose and 3D information to optimize re-projection error or rendered RGB error for enhanced results. 

However, the separation of camera parameter estimation and scene reconstruction in such approaches often leads to suboptimal results since only precise reconstructions of both 3D scenes and camera parameters can result in novel views that closely approximate ground truth. This separation not only lengthens the pipeline but also introduces a compounding effect: errors in camera parameters estimation can deteriorate scene reconstruction, leading to additional inaccuracies in camera parameters estimation \citep{noposplat}. 

In this study, we introduce \textbf{{\methodname}} to jointly estimate camera camera parameters of input views and synthesize novel views for scenes by directly connecting camera parameters with 3D Gaussians within a single network as shown in \cref{fig: teaser}.
Inspired by GeoLRM \citep{geolrm}, LeanGaussian \citep{wu2024dig3d} and UNIG \citep{unig}, which leverage the Deformable DEtection TRansformer (Deformable DETR) \citep{DETR} framework in 3D, we introduce separate queries for 3D Gaussians and camera parameters (3D queries and camera queries in short). The 3D queries are utilized to generate 3D Gaussians by multi-layer perceptron (MLP), while the camera queries are employed to regress camera extrinsics and intrinsics using our introduced \textbf{{\rayref}} (2D reference point determined rays) representation on the camera queries. 3D Gaussians (center of them) are then projected onto each view using camera parameters, and the projected points are regarded as 2D reference points in deformable cross-attention. Subsequently, both 3D and camera queries are updated by deformable cross-attention with image features surrounding these 2D reference points. The above process is called  \textbf{{\camdfa}} (camera-aware multi-view deformable cross-attention), which efficiently establishes a connection between 3D Gaussians and camera parameters through 2D reference points, updating them effectively.

Although the model effectively optimizes 3D Gaussians and camera parameters jointly, a \textbf{challenge} arises when regressing camera parameters directly from numerous parameters (resulting from multiplication of the number of reference points by hidden dimension). This approach may not be optimal for neural learning, which typically benefits from over-parameterized distributed representations \citep{rays}. Inspried by Ray-Diffusion \citep{rays}, which represents camera parameters as rays from camera center to each pixel on a image, we design \textbf{{\rayref}}, modeling camera parameters as rays from camera center to the 2D reference points on each input view and design an overdetermined equation system using equations of the rays to solve camera parameters. With the aforementioned design, our model solves both camera extrinsics and intrinsics through RQ-decomposition, enabling scene reconstruction without the need for any input view camera parameters.

With the aforementioned design, our model effectively establishes a direct connection between 3D Gaussians and camera parameters through a \textbf{shared set} of 2D reference points, enabling joint optimization. The collaboration of 3D Gaussians and camera parameters plays an important role in the 3D Gaussian reconstruction and NVS outcomes. 

In summary, our contributions are as follows:
\begin{itemize}
    \item We propose a novel approach to modeling 3D Gaussians and camera poses of input views jointly without any pre-processing or post-processing steps, making scene reconstruction without any input view camera parameters.
    \item We establish a direct connection between 3D Gaussians and camera parameters through 2D reference points and using \textbf{{\camdfa}} and \textbf{{\rayref}} to optimize both 3D Gaussians and camera parameters possible. 
    \item Our method outperforms previous methods, both pose-required and pose-free, on RealEstate10K (RE10K) \citep{re10k} and ACID \citep{ACID} within the same pose-free setting.
\end{itemize}

\section{Related Work} \label{sec: related_works}
\begin{figure*}
    \centering
    \includegraphics[width=0.9\textwidth]{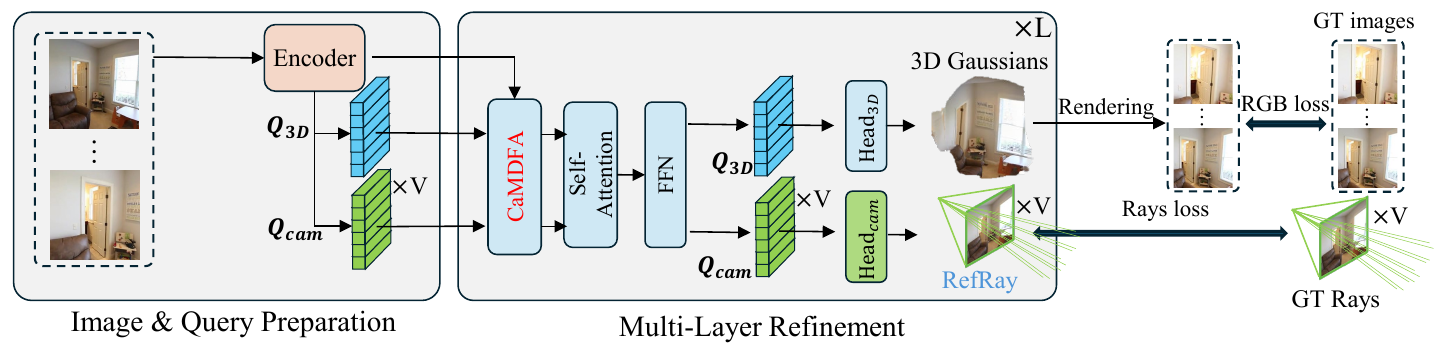}
    
    \caption{Overall Framework: Our model employs an encoder-decoder architecture to simultaneously reconstruct 3D Gaussians and camera parameters. The Vision Transformer (ViT) encoder processes all input images, with the resulting image features serving as keys and values in \textcolor{red}{{\camdfa}} block (\cref{sec: dfa}) within the decoder layer. In our approach, define queries for 3D Gaussians and camera queries separately, incorporating them into the deformable decoder layer alongside the image features. Following $L$ decoder layers, the model produces queries for 3D Gaussians and cameras, which are used to go through a fully connected network (FFN) operation to regress the 3D Gaussians and \textcolor{blue}{{\rayref}} (in Pl\"{u}cker coordinates \citep{plucker} denoting the direction and moment \cref{sec: rays2cam}) that points from camera center to reference points. Subsequently, the predicted 3D Gaussians undergo rendering via Gaussian Splatting \citep{3d-gs} to generate the novel views while the rays are utilized to solve the camera parameters in \cref{sec: rays2cam}.}
    
    \label{fig: overview}
\end{figure*}
\noindent\textbf{Pose required scene reconstruction NVS}
The field of deep learning-based NVS has seen rapid advancement recently, driven by differentiable rendering techniques like NeRF \citep{nerf} and 3DGS \citep{3d-gs}. Subsequent methods, such as \cite{nerfsdf, mipnerf, fastnerf, meshnerf, pixelnerf, instantnerf, instant3d, instantmesh, triplane-gs, gslrm, geolrm, LGM, wonder3d, mvsplat, pixelsplat, wu2024dig3d, unig, grm, nice, nicer, gsslam}, have introduced novel deep learning models for this task. While some approaches take a single-view image as input but lack finer details, others accept multi-view images but necessitate ground truth camera poses, limiting real-world applicability. Methods like LGM and InstantMesh \citep{LGM, instantmesh} leverage single-view inputs and utilize diffusion models \citep{stablediffusion, imagedream, mvdream} to generate fixed views (e.g., front, right, back, left) before applying their models, exploiting the ease of obtaining poses for these predetermined views. However, the generated views may introduce uncertainties and distortions compared to real-world scenarios due to the reliance on diffusion models. Some methods leverage off-the-shelf SfM techniques \citep{colmap, sfm1, sfm2} to estimate input view poses. Nevertheless, this approach demands a denser array of views to achieve comparatively accurate estimations, and also increase the complexity of the process.

\noindent\textbf{Pose-free scene reconstruction and NVS}
\cite{reconx} do not need pose input for scene NVS but instead require per-scene optimization.
Pose-free methods without per-scene optimization can generally be categorized into two groups. The first group comprises methods \citep{nerf--, dbarf, posefree, sparp, unifying, flowcam, nope-nerf, mip, PF3PLAT, scene_transformer} that optimize pose and scene simultaneously. However, these methods typically employ separate networks to estimate them, often lacking direct connections between them or relying on initializations from other pose and scene prediction models. This can lead to complex pipelines and a lack of explicit connections between pose and scene. The second category includes methods \citep{dust3r, mast3r, noposplat, splatt3r, monst3r} which initially reconstruct the scene in a canonical space  without explicit pose input, rely on point or feature matching for scene reconstruction. Subsequently, the PnP algorithm is used to compute camera poses by matching points reconstructed from different input views, requiring sufficient overlap between input views and do pose-processing for pose estimation. Techniques for object reconstruction \citep{pf-lrm, forge, leap} often involve lifting each pixel to a 3D Gaussian or relying on voxel representations, limiting input resolution and efficiency. In contrast, our model optimizes scene and pose jointly without additional pipelines.
The comparison between our method and representative previous methods is in \cref{tab: method_comparison_simple} with more methods comparison in \cref{app: method_comparison} \cref{tab:method_comparison}.

\noindent\textbf{Model-free pose estimation}
Methods like \cite{relpose, relpose++, rays, FoundationPose, posediffusion, BundleSDF} directly regress poses from input images without 3D models. Many of these techniques leverage tools such as cross-attention and cost-volume to compare variations and similarities between different views and utilize the matching outcomes to estimate camera poses. However, lacking the assistance of 3D models, these methods are generally less effective compared to those incorporating 3D models. Some methods \citep{nocs, nocsformer} utilize Normalized Object Coordinate Space (NOCS) for pose estimation. However, the reliance on well-aligned datasets that defines the positive direction for each object restricts its applicability.

\begin{table}[!htp]
\centering

\caption{Comparison with representative methods. `Pose-Free' means that the model is pose-free, `Jointly Opt' means that the model optimize camera and 3D jointly, `Cano. Space' means that the model define 3D points in the canonical space for all views, `Intri.-Free' means that the model does not require intrinsic input, `3D GS' means the model use 3D Gaussians as 3D representation.}

\label{tab: method_comparison_simple}

\begin{tabular}{lccccc}
\toprule
 Method & Pose-Free & Jointly Opt. & Cano. Space & Intri.-Free & 3D GS \\
\midrule
MVSplat & \XSolidBrush & \XSolidBrush & \XSolidBrush & \XSolidBrush & \Checkmark\\
CoPoNeRF & \Checkmark & \Checkmark & \Checkmark & \XSolidBrush & \XSolidBrush \\
SelfSplat & \Checkmark & \XSolidBrush & \XSolidBrush & \XSolidBrush & \Checkmark\\
NoPoSplat & \Checkmark & \XSolidBrush & \Checkmark & \XSolidBrush & \Checkmark\\
Ours  & \Checkmark & \Checkmark & \Checkmark & \Checkmark & \Checkmark \\
\bottomrule

\end{tabular}

\end{table}
\section{Methods}
The details of our proposed method are given in this section. First, the preliminaries for 3D GS are briefly described in \cref{sec: 3d-gs}. Then, an overview of {\methodname} is in \cref{sec: overview}.
After that, Camera-aware Multi-view DeFormable cross-Attention ({\camdfa}) module, the core contribution of {\methodname} is described in \cref{sec: dfa} in detail.
Moreover, the definition of 2D refernce point determined rays {\rayref}, which is proposed to represent the camera parameters in {\methodname}, is introduced in \cref{sec: rays2cam}.
Finally, the training objective is introduced in \cref{sec: loss}.


\subsection{Prelimenaries of 3D GS} \label{sec: 3d-gs}

3D GS \citep{3d-gs} is a rendering method that utilize 3D Gaussians to represent 3D objects or scenes. 
In this context, a scene which consists of $N$ 3D Gaussians can be represented as a set of parameters $\{\mathrm{\textbf{SH}}, \boldsymbol{\mu}, \boldsymbol{\sigma}, \mathbf{R}, \mathbf{S}\}$.
The color of a 3D Gaussian ellipsoid is represented by spherical harmonics (\textbf{SH}) $\in \mathbb{R}^{N \times k}$ \citep{sh, 3d-gs}, with degree of freedom $k$, while the geometry is described by position $\boldsymbol{\mu} \in \mathbb{R}^{N \times 3}$, shape (including rotation $\mathbf{R} \in \mathbb{R}^{N \times 4}$ in quternion and scales $\mathbf{S} \in \mathbb{R}^{N \times  3}$), and opacity $\boldsymbol{\sigma} \in \mathbb{R}^N$ \citep{EWASplatting, 3d-gs}. Therefore, the objective of 3D Gaussian reconstruction is to estimate the aforementioned parameters for 3D Gaussians.

\subsection{Overview of {\methodname}} \label{sec: overview}
\noindent\textbf{Image feature preparation.} \label{sec: encoder}
Following \cite{dust3r, mast3r, noposplat}, we utilize the Vision Transformer (ViT) \citep{vit} to extract multi-view image features from input images. The input images are first patchified and flattened to sequences of image tokens, which are then sent into ViT layers to derive the multi-view image features $\textbf{F}$. To expedite convergence, we initialize our ViT encoder based on NoPoSplat\cite{noposplat}.

\noindent\textbf{Query preparation for 3D Gaussians and camera parameters.} \label{sec: query_prepare}
As mentioned in \cref{sec:intro}, our model define separate queries for 3D Gaussians and camera parameters as 3D queries $\textbf{Q}_{3D} \in \mathbb{R}^{N \times D}$ and camera queries $\textbf{Q}_{cam} \in \mathbb{R}^{V \times N \times D}$ respectively, where $N$ denotes the number of queries, which is also the number of 3D Gaussians, V is the number of input views. Note that both quereis share the same locations as the centers of 3D Gaussians and therefore share the same set of 2D reference points after projection. The 3D Gaussian queries' corresponding 3D Gaussian parameters $\textbf{G} \in \mathbb{R}^{N \times K}$ are obtained by conducting an MLP on $\textbf{Q}_{3D}$, where $K$ denotes the dimention of 3D Gaussian parameters. The camera queries' corresponding {\rayref} $\boldsymbol{\cal{R}} \in \mathbb{R}^{V \times N \times 6}$, which is the intermediate representation of camera parameters (see \cref{sec: rays2cam}), are obtained by conducting MLP on $\textbf{Q}_{cam}$. 
The 3D quereis are initialized from the flattening image features from ViT (see \cref{app: init} for more details). The camera queries are randomly initialized as learnable vectors.


\noindent\textbf{Multi-layer refinement.} \label{sec: decoder}
As shown in \cref{fig: overview}, within the multi-layer refinement, our method focuses on the optimization of 3D Gaussians $\textbf{G}$ and camera {\rayref} $\boldsymbol{\cal{R}}$ in each layer. As shown in \cref{fig: overview}, 3D queries and camera queries first pass through a camera-aware multi-view deformable cross-attention ({\camdfa}) block, where the multi-view image features from the ViT encoder contribute as keys and values, as detailed in \cref{sec: dfa}. Subsequently, the updated 3D queries and camera queries conduct self-attention individually. In particular, to mitigate memory costs, we implement a spatially efficient self-attention strategy inspired by UniG\citep{unig} (for more details, please refer to \cref{app: self_attn}. After that, the queries are updated by feed-forward networks (FFNs). The progress for updating queries in the $l$-th layer can be formulated as
\begin{gather} \label{eq: decoder_layer}
    \mathbf{Q}_{3D}^{l}, \mathbf{Q}_{\text{cam}}^{l} = \mathrm{{\camdfa}}(\mathbf{Q}_{3D}^{l-1}, \mathbf{Q}_{\text{cam}}^{l-1}, \mathbf{P}^l, \mathbf{F}), \\
    \mathbf{Q}_{3D}^{l} = \mathrm{FFN}(\mathrm{SA}(\mathbf{Q}_{3D}^{l})), \quad \mathbf{Q}_{\text{cam}}^{l} = \mathrm{FFN}(\mathrm{SA}(\mathbf{Q}_{\text{cam}}^{l})) \nonumber
\end{gather}
where $\mathrm{SA}$ representing the self-attention layer and $\mathbf{P}^l$ indicates 2D reference points in the $l$-th layer (refer to details in \cref{sec: dfa}).
Finally, the updated 3D queries are utilized to refine the 3D Gaussians by \cref{eq: G_update},
\begin{equation} \label{eq: G_update}
    \mathbf{G}^{l} = \Delta \mathbf{G}^{l} + \mathbf{G}^{l-1}, \Delta \mathbf{G}^{l} = \mathrm{Head_{3D}}(\mathbf{Q}_{3D}^{l})
\end{equation}
where $\mathbf{G}^{l}$ is the parameters of 3D Gaussians in current layer \footnote{In the first layer, $\mathbf{G}^{l-1}$ is initialized by conducting an MLP on the intialized 3D queries.},  $\mathrm{Head_{3D}}$ is an MLP. 
Meanwhile, camera queries are utilized to update camera parameters, with the intermediate representation {\rayref} (see \cref{sec: rays2cam}) as shown in \cref{eq: ray_update},
\begin{equation} \label{eq: ray_update}
    \boldsymbol{\cal{R}}^{l} = \Delta \boldsymbol{\cal{R}}^{l} + \boldsymbol{\cal{R}}^{l-1}, \Delta \boldsymbol{\cal{R}}^{l} = \mathrm{Head_{cam}}(\mathbf{Q}_{cam}^{l})
\end{equation}
where $\boldsymbol{\cal{R}}^{l}$ denotes RefRay in current layer \footnote{In the first layer, $\boldsymbol{\cal{R}}^{l-1}$ is initialized from initial camera parameters, which are defined as identity matrics.}, and $\mathrm{Head_{cam}}$ is an MLP.
After that the camera parameters $\textbf{K}^{l}$ and $\boldsymbol{\pi}^{l}$ are updated by the ray representation {\rayref} $\boldsymbol{\cal{R}}^{l}$ (see \cref{sec: rays2cam} and \cref{app: proof}), and then they are passed into the next layer for 2D reference projection as in \cref{eq: pinhole}.

\subsection{\textbf{{\camdfa}}} \label{sec: dfa}
\noindent\textbf{2D reference points in \textbf{{\camdfa}}.}
As mentioned in \cref{sec: overview}, 2D reference points in {\camdfa} are defined as the projected points of 3D Gaussians.
Specifically, 2D reference points on the $i$-th view are generated by projecting the center of 3D Gaussians onto the $i$-th view, using the pinhole camera model \citep{pinhole1, pinhole2}, as shown in \cref{fig: dfa} and \cref{eq: pinhole},
\begin{equation} \label{eq: pinhole}
    \mathbf{P}_i = \textbf{K}_i\boldsymbol{\pi}_i\boldsymbol{\mu}
\end{equation}
where $\mathbf{P}_i$ represents the 2D reference points on the $i$-th view, $\boldsymbol{\pi}_i$ is the camera extrinsics, and $\textbf{K}_i$ is camera intrinsics of the $i$-th view. 

\noindent\textbf{Query updating for 3D Gaussian in {\camdfa}.}
The locations of the 3D Gaussian queries are defined in canonical space \footnote{The canonical space is defined as one input view's local camera coordinates \citep{noposplat}, and the selected input view is called reference view.} unitarily rather than individually for each view. Sharing the queries across multi-views directly is impractical. Inspired by \cite{unig}, we consider incorporating camera information for each view. Specifically, we employ camera modulation with adaptive layer norm (adaLN) \citep{openlrm, stylegan, stylegan1, stylegan2} to derive 3D queries for the $i$-th view as shown in \cref{eq: modulation}, 
\begin{gather} \label{eq: modulation}
    \textbf{Q}_{3D_i} = \mathrm{LayerNorm}(\textbf{Q}_{3D}) * (1 + \text{scale}) + \text{shift}, \\ \text{shift}, \text{scale} = \mathrm{MLP}(\textbf{Q}_{cam_i}) \nonumber
\end{gather} 
The 3D queries are refined by conducting deformable cross-attention (DFA) on it with image features sampled surrounding the 2D reference points $\mathbf{P}_i$ (see \cref{eq: pinhole}) in the $i$-th view as keys and values as shown in \cref{eq: dfa},
\begin{equation} \label{eq: dfa}
    \textbf{Q}_{3D_i} = \mathrm{DFA}(\textbf{Q}_{3D_i}, \mathbf{P}_i, \textbf{F}_i)
\end{equation}
where $\textbf{F}_i$ is the image feature for the $i$-th view, $\mathrm{DFA}$ is the deformable cross-attention (see details in \cref{app: model_details}).
After refining the 3D queries for each view, it is essential to merge them to derive the unitary 3D queries. In alignment with \cite{unig}, we utilize a weighted sum by \cref{eq: weight_sum}.
\begin{equation} \label{eq: weight_sum}
    \textbf{Q}_{3D} = \sum_i w_i\textbf{Q}_{3D_i}, w_i=\mathrm{sigmoid}(\mathrm{MLP}(\textbf{Q}_{3D_i})).
\end{equation}


\noindent\textbf{Query updating for camera parameters in \textbf{{\camdfa}}.}
Instead of defining in canonical space like the 3D Gaussian queries, the camera queries are individually defined for each view. Meanwhile, the pose of each camera is specified relative to the reference view (defined as the $1$-th view in our scenario without loss of generality). Therefore, to update camera queries for the $i$-th view, we only consider the $1$-th and the $i$-th view as shown in \cref{cam_query}.
\begin{gather} \label{cam_query}
    \textbf{Q}_{cam_i} \Leftarrow \mathrm{DFA}(\textbf{Q}_{cam_i}, \mathbf{P}_{{cam}_i}, \textbf{F}_i) \\
    \textbf{Q}_{cam_1} \Leftarrow \mathrm{DFA}(\textbf{Q}_{cam_1}, \mathbf{P}_{{cam}_1}, \textbf{F}_i). \nonumber
\end{gather}
After refining $\textbf{Q}_{cam_1}$ and $\textbf{Q}_{cam_i}$, our model performs an extra dense cross-attention to fuse them as shown in \cref{fusion_cam},
\begin{gather} \label{fusion_cam}
    \textbf{Q}_{cam_i} \Leftarrow \textbf{A}_i \cdot \textbf{W}^V\textbf{Q}_{cam_1}, \\
    \textbf{A}_i = \text{softmax}\left(\frac{\textbf{W}^Q\textbf{Q}_{cam_i} \cdot (\textbf{W}^K\textbf{Q}_{cam_1})^T}{\sqrt{d_k}}\right) \nonumber
\end{gather}
where $\textbf{W}^Q, \textbf{W}^k, \textbf{W}^V$ denotes transformation matrices, and $d_k$ is the dimension of keys.

\begin{figure}
    \centering
    \includegraphics[width=0.7\textwidth]{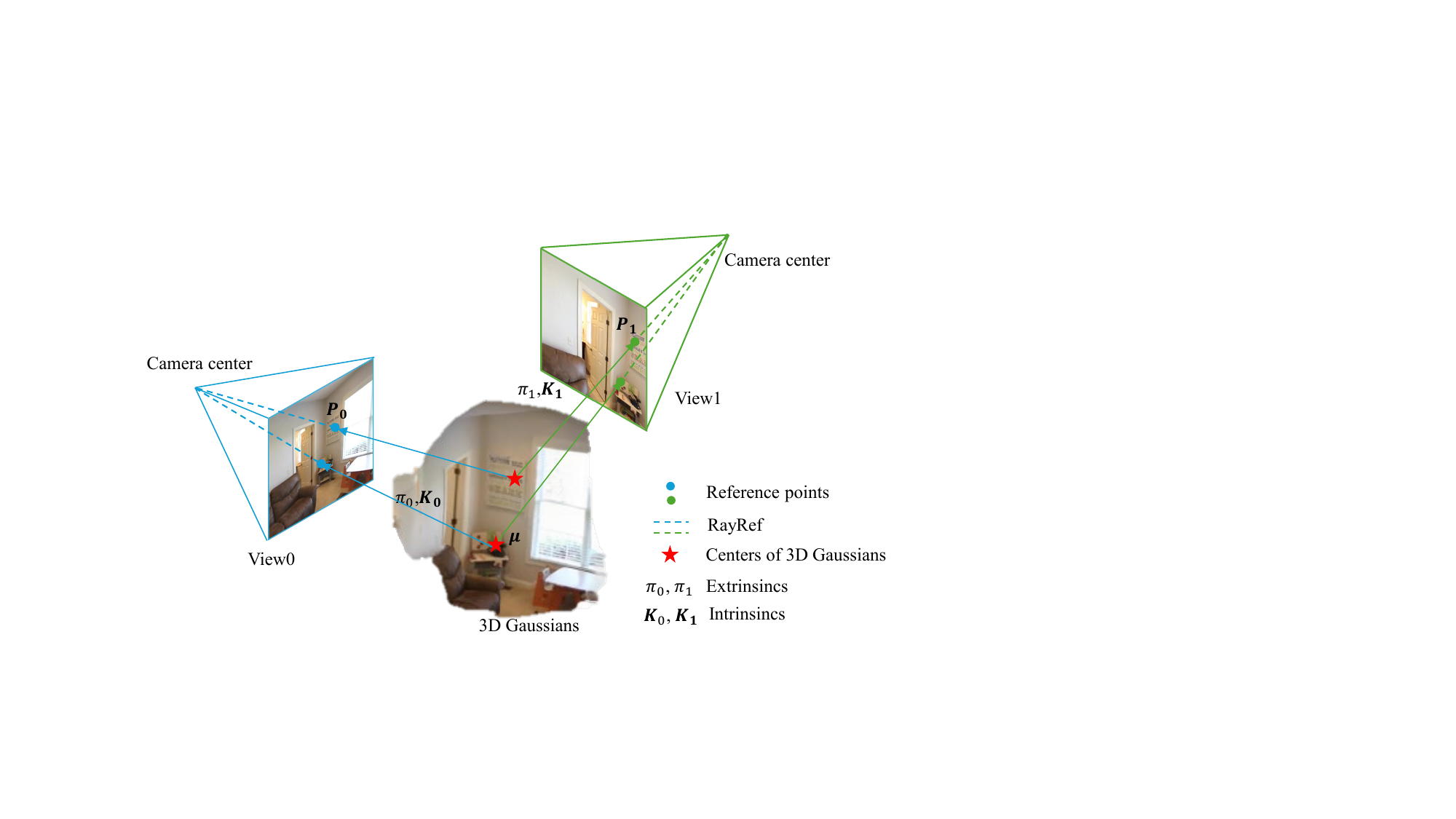}
    
    \caption{Our model projects the center of 3D Gaussians onto each input view by the camera parameters to get the reference points. The reference points are utilized to generate {\rayref} described in \cref{sec: rays2cam}, which are the rays from the camera center to reference points. The 3D Gaussians and camera parameters are then explicitly linked by the reference points.}
    \label{fig: dfa}
    
\end{figure}


\noindent\textbf{Insights in {\camdfa}.}
The objective of {\camdfa} is to link 3D Gaussians and camera parameters explicitly by sharing the same set of the projected 2D reference points. With the constraint in \cref{eq: pinhole}, 3D Gaussians (with the centers $\boldsymbol{\mu}$) and camera parameters ($\textbf{K}$, $\boldsymbol{\pi}$) are optimized simultaneously. This joint optimization aids in obtaining more meaningful 2D reference points and subsequently extracting more relevant image features from multi-view images. 

\subsection{{\rayref} and Pl\"{u}cker-Camera sparse mapping}  \label{sec: rays2cam}
Typically, camera for the $i$-th view is parameterized by extrinsics $\boldsymbol{\pi}_i$ (consisting of rotations $\boldsymbol{\pi}_{i_\textbf{R}} \in SO(3)$ and translations $\boldsymbol{\pi}_{i_\mathbf{t}} \in \mathbb{R}^3$) and intrinsics $\textbf{K}_i \in \mathbb{R}^{3 \times 3}$. However, directly regressing $\textbf{K}_i$ and $\boldsymbol{\pi}_i$ from the camera queries ($\textbf{Q}_{\text{cam}_i} \in \mathbb{R}^{N \times D}$ for the $i$-th view), which entails regressing only 16 parameters (9 for the rotation matrix, 3 for translation, and 4 for intrinsics including focal lengths and principal points) from the $N \times D$ parameters through an MLP, may be suboptimal for neural learning~\cite{rays}. 
Therefore, inspired by \cite{rays}, we represent the camera parameters by $N$ rays, transferring starting from the camera center to the $N$ 2D reference points on the image plane as shown in \cref{fig: dfa}. 
More specifically, our approach over-parameterize camera in the $i$-th view by a set of rays $\boldsymbol{\cal{R}}_i = {\textbf{r}_{i1}, \ldots, \textbf{r}_{iN}}$, where each ray $\textbf{r}_{ij} \in \mathbb{R}^6$ is represented by Pl\"{u}cker coordinates \citep{plucker}: $\textbf{r}_{ij} \in \mathbb{R}^6$. 
Different from previous methods \citep{rays}, we define the rays sparsly on the 2D refernce points instead of on each pixel to enhance the relationship between 3D Gaussians and camera parameters. After each layer, once the rays are predicted by \cref{eq: ray_update}, the camera parameters can be sovled by RQ-decomposition on a overdetermined system of equations from the Pl\"{u}cker coordinates of {\rayref} with details of  Pl\"{u}cker-camera sparse mapping in \cref{app: proof}.

\subsection{Training objective} \label{sec: loss}
We leverage the differentiable rendering method 3D GS \cite{3d-gs} to generate RGB images from the 3D Gaussians $\textbf{G}$ predicted by our model. Aligning with \cite{openlrm, LGM, noposplat}, we employ an RGB loss in \cref{eq: loss_3d}, which consists of both a mean square error loss $\mathcal{L}_{\text{MSE}}$ and a VGG-based LPIPS (Learned Perceptual Image Patch Similarity) loss \citep{LIPIS} $\mathcal{L}_{\text{LPIPS}}$, to guide the rendered views. Here $I_{pd}$ and $I_{gt}$ represent rendered views and ground truth images, ewspectively.
\begin{equation} \label{eq: loss_3d}
    \mathcal{L}_{3D} = \mathcal{L}_{\text{MSE}}(I_{pd}, I_{gt}) + \lambda\mathcal{L}_{\text{LPIPS}}(I_{pd}, I_{gt})
\end{equation}
For the camera parameters, with the 2D reference points, we can calculate the rays from the ground truth camera center to the $N$ 2D reference points to get the ray supervision ${\cal R}_{j_{gt}}$. The ray loss is defined as \cref{eq: loss_rays}. 
\begin{equation} \label{eq: loss_rays}
{\cal L}_{rays} = \sum_{j=1}^N \left\| {\cal R}_{j_{gt}} - {\cal R}_{j_{pd}} \right\|_2.
\end{equation}
The total loss can be described as \cref{eq: loss}, where $\lambda_{rays}$ and $\lambda_{3D}$ are weights of 3D and ray loss, respectively.
\begin{equation} \label{eq: loss}
    \mathcal{L} = \lambda_{rays}\mathcal{L}_{rays} + \lambda_{3D}\mathcal{L}_{3D}
\end{equation}

\section{Experiments} \label{sec: exp}
\begin{table*}[!htp]
\centering
\caption{Novel view synthesis performance comparison on the ACID \citep{ACID} dataset. `-A' means with evaluation-time pose alignment.}

\label{tab:acid}

\resizebox{0.95\linewidth}{!}{\begin{tabular}{lcccc}
\toprule
 & \multicolumn{4}{c}{Overlap Settings} \\
\cmidrule(lr){2-5}
Method & Small & Medium & Large & Average \\
& PSNR\(\uparrow\) SSIM\(\uparrow\) LPIPS\(\downarrow\) & PSNR\(\uparrow\) SSIM\(\uparrow\) LPIPS\(\downarrow\) & PSNR\(\uparrow\) SSIM\(\uparrow\) LPIPS\(\downarrow\) & PSNR\(\uparrow\) SSIM\(\uparrow\) LPIPS\(\downarrow\) \\
\midrule
\textbf{Pose-Required} & & & & \\
pixelNeRF & 19.376 0.535 0.564 & 20.339 0.561 0.537 & 20.826 0.576 0.509 & 20.323 0.561 0.533 \\
AttnRend & 20.942 0.616 0.398 & 24.004 0.720 0.301 & 27.117 0.808 0.207 & 24.475 0.730 0.287 \\
pixelSplat & 22.053 0.654 0.285 & 25.460 0.776 0.198 & 28.426 0.853 0.140 & 25.819 0.779 0.195 \\
MVSplat & 21.392 0.639 0.290 & 25.103 0.770 0.199 & 28.388 0.852 0.139 & 25.512 0.773 0.196 \\
\midrule
\textbf{Pose-Free} & & & & \\
DUSt3R & 14.494 0.372 0.502 & 16.256 0.411 0.453 & 17.324 0.431 0.408 & 16.286 0.411 0.447 \\
MASt3R & 14.242 0.366 0.522 & 16.169 0.411 0.463 & 17.270 0.430 0.423 & 16.179 0.409 0.461 \\
CoPoNeRF & 18.651 0.551 0.485 & 20.654 0.595 0.418 & 22.654 0.652 0.343 & 20.950 0.606 0.406 \\
NoPoSplat & 21.987 0.603 0.278 & 23.236 0.689 0.238 & 24.468 0.746 0.180 & 23.416 0.692 0.226 \\
NoPoSplat-A & 23.087 0.685 0.258 & 25.624 0.777 0.193 & 28.043 0.841 0.144 & 25.961 0.781 0.189 \\
Ours & 22.809 0.658 0.228 & 23.292 0.707 0.218 & 24.625 0.783 0.172 & 23.648 0.723 0.204 \\
\textbf{Ours-A} & \textbf{24.343 0.713 0.237} & \textbf{25.933 0.798 0.187} & \textbf{28.467 0.859 0.120} & \textbf{26.484 0.803 0.174} \\
\bottomrule
\end{tabular}}

\end{table*}

\begin{figure}
    \centering
\includegraphics[width=0.8\textwidth]{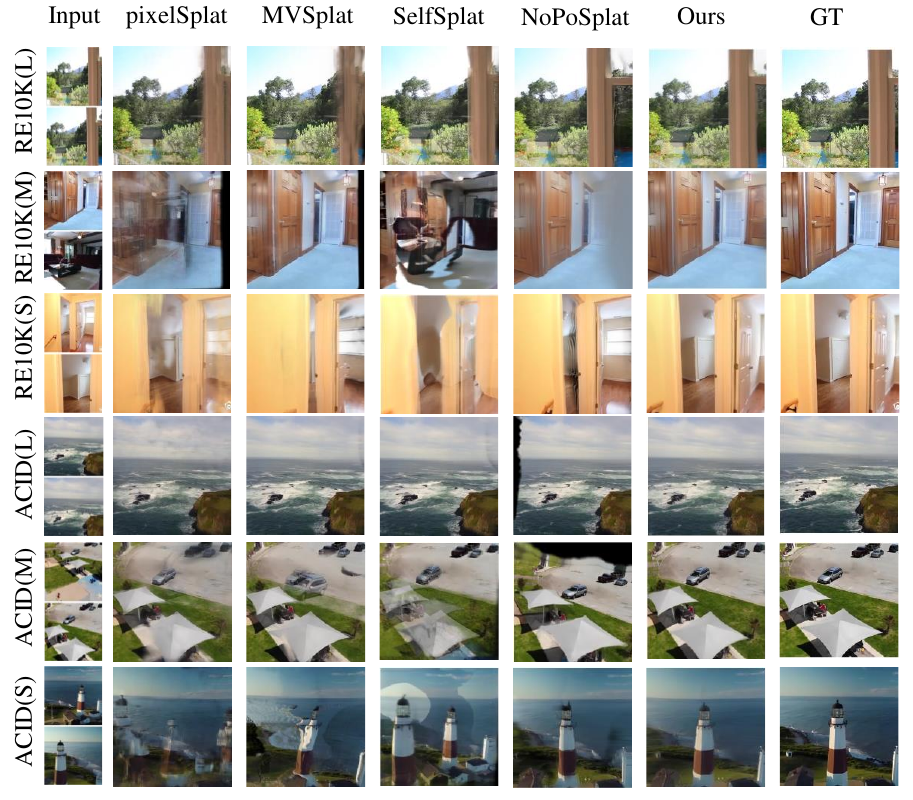}

    \caption{Qualitative comparison on RE10K and ACID datasets. the `L', `M', and  `S' in the brackets meaning the groups of overlapping large, medium, and small, respectively.}
    \label{fig: vis_all}
    
\end{figure}
\noindent\textbf{Datasets}
Our paper focuses on benchmarking the Novel View Synthesis (NVS) task using the entire scene dataset. We have aligned our approach with the datasets employed in prior works \cite{noposplat, mvsplat, pixelsplat}, which leverage RealEstate10K (RE10K) \citep{re10k} and ACID \citep{ACID} predominantly for both training and testing. We have adhered to the same train-test division and utilized the associated view indices outlined in \cite{noposplat}. Additionally, to enhance the scalability of our model, we have followed the approach in \cite{noposplat} by incorporating the combined datasets of RE10K and DL3DV \citep{dl3dv} (including all 11K subsets) for model training.

\noindent\textbf{Evaluation metrics}
Our model optimizes pose-free NVS and pose estimation simultaneously. For NVS evaluation, we consider metrics such as PSNR, SSIM, and LPIPS \citep{lpips}. In terms of pose estimation, we calculate the Area Under the Curve (AUC) for cumulative pose error using thresholds of \ang{5}, \ang{10}, and \ang{20} \citep{noposplat, auc}. Furthermore, we also evaluate the performance of our model using input views with varying degrees of overlap, where a larger overlap indicates closer proximity between the input views. We categorize the overlaps into three groups: small (5\% - 30\%), medium (30\% - 55\%), and large (55\% - 80\%) following \citep{noposplat}.

\noindent\textbf{Baselines}
We conduct comparison to the pose-free method NoPoSplat \citep{noposplat}, and also other sparse view reconstruction methods, including pose-required methods pixelNeRF \citep{pixelnerf}, AttnRend \citep{AttnRend}, pixelSplat \citep{pixelsplat}, and MVSplat \citep{mvsplat}, and pose-free methods DUSt3R \citep{dust3r}, MASt3R \citep{mast3r}, Splatt3R \citep{splatt3r}, SelfSplat \citep{selfsplat}, CoPoNeRF \citep{CoPoNeRF}, and RoMa \citep{roma}.

\noindent\textbf{Implementation details}
We employ ViT initialized by NoPoSplat and multi-layer refinement contains 4 layers. In accordance with prior research, we input all images at a resolution of 256 $\times$ 256. To compare fairly, we select the number of 3D Gaussians being same to previous methods, that is, the number of pixels in 2 views ($2 \times 256 \times 256$). Our model is trained on 8 A100 GPUs for 1 day.

\subsection{Results and analysis}

\begin{table*}[!t]
\centering
\caption{Novel view synthesis performance comparison on the RealEstate10k \citep{re10k} dataset. `-A' means with evaluation-time pose alignment. `-IF' means the model is intrinsic free.}

\label{tab:re10k}
\resizebox{0.95\linewidth}{!}{
\begin{tabular}{lcccc}
\toprule
 & \multicolumn{4}{c}{Overlap Settings} \\
\cmidrule(lr){2-5}
Method & Small & Medium & Large & Average \\
& PSNR\(\uparrow\) SSIM\(\uparrow\) LPIPS\(\downarrow\) & PSNR\(\uparrow\) SSIM\(\uparrow\) LPIPS\(\downarrow\) & PSNR\(\uparrow\) SSIM\(\uparrow\) LPIPS\(\downarrow\) & PSNR\(\uparrow\) SSIM\(\uparrow\) LPIPS\(\downarrow\) \\
\midrule
\textbf{Pose-Required} & & & & \\
pixelNeRF & 18.417 0.601 0.526 & 19.930 0.632 0.458 & 20.869 0.639 0.485 & 19.824 0.626 0.485 \\
AttnRend & 19.151 0.663 0.368 & 22.532 0.763 0.186 & 25.897 0.845 0.269 & 22.664 0.762 0.269 \\
pixelSplat & 20.263 0.717 0.266 & 23.711 0.809 0.181 & 27.151 0.879 0.122 & 23.848 0.806 0.185 \\
MVSplat & 20.353 0.724 0.250 & 23.778 0.812 0.173 & 27.408 0.884 0.116 & 23.977 0.811 0.176 \\
\midrule
\textbf{Pose-Free} & & & & \\
DUSt3R & 14.101 0.432 0.468 & 15.419 0.451 0.402 & 16.427 0.453 0.432 & 15.382 0.447 0.432 \\
MASt3R & 13.534 0.407 0.494 & 14.945 0.436 0.418 & 16.028 0.444 0.452 & 14.907 0.431 0.452 \\
Splatt3R & 14.352 0.475 0.472 & 15.529 0.502 0.425 & 15.817 0.483 0.421 & 15.318 0.490 0.436 \\
CoPoNeRF & 17.393 0.585 0.462 & 18.813 0.616 0.392 & 20.464 0.652 0.318 & 18.938 0.619 0.388 \\
SelfSplat & 17.506 0.550 0.461 & 19.357 0.704 0.378 & 20.868 0.672 0.256 & 19.33	0.656 0.363 \\
NoPoSplat & 21.814 0.765 0.220 & 23.044 0.787 0.178 & 25.408 0.844 0.126 & 23.424 0.798 0.173 \\
NoPoSplat-A & 22.514 0.784 0.210 & 24.899 0.839 0.160 & 27.411 0.883 0.119 & 25.033 0.838 0.160 \\
Ours-IF & 21.312 0.741 0.234 & 22.891 0.773 0.189 & 25.312 0.829 0.138 & 23.200 0.781 0.185 \\
Ours & 22.843 0.781 0.192 & 23.394 0.804 0.175 & 26.270 0.872 0.110 & 24.093 0.818 0.160 \\
\textbf{Ours-A} & \textbf{24.117 0.827 0.177} & \textbf{25.018 0.859 0.159} & \textbf{27.985 0.904 0.110} & \textbf{25.656 0.864 0.149} \\
\bottomrule
\end{tabular}
}
\end{table*}

\noindent\textbf{Quantity results for NVS}
The comparison results for NVS performance on RE10K and ACID dataset are depicted in \cref{tab:re10k} and \cref{tab:acid}. As demonstrated in the table, our model consistently outperforms previous methods, even surpassing methods that rely on ground-truth pose input, particularly excelling in scenarios with limited view overlap. In particular, our unposed image model successfully reconstructs a plausible 3D scene that aligns well with the inputs provided. However, reconstructing a 3D scene with only two input views inherently poses challenges due to the ambiguity arising from multiple scenes producing the same pair of images \citep{noposplat}. To address this, as discussed in \cite{noposplat, instantsplat, nerf--}, we adopt an evaluation-time pose alignment strategy to refine the camera pose for the target view, ensuring a fair comparison. For each evaluation sample, we initially reconstruct 3D Gaussians using our proposed approach. Subsequently, we fine-tune the target camera pose to align the rendered image closely with the ground-truth image. We present results for both models with and without this alignment, denoting the aligned method with `-A'. 
As outlined in \cref{sec: rays2cam}, rays are capable of representing not only camera poses but also intrinsics. Consequently, we conduct experiments that do not provide intrinsics but instead predicted them, the results are represented by `Ours-IF' in \cref{tab:re10k}. The results is worse than using ground-truth intrinsics but comparable to previous methods.

\noindent\textbf{Quality results for NVS}
We present quality comparisons with pose-required methods (pixelSplat \citep{pixelsplat} and MVSplat \citep{mvsplat}) and the two categories of pose-free methods mentioned in \cref{sec:intro}. The first category is that solely optimize 3D while neglecting poses (NoPoSplat \citep{noposplat}) and the second category is that optimize both pose and 3D using separate networks (SelfSplat \citep{selfsplat}) \cref{fig: vis_all} and \cref{app: more_vis}. Our method demonstrates superior performance compared to previous ones. Pose-required methods typically employ a `transform-then-fuse' strategy \citep{noposplat}, generating 3D Gaussians under each view's camera coordinates and then concatenating them in world coordinates, resulting in blurring in novel views. Pose-free methods that focus solely on 3D may encounter issues with incorrect rendering in novel views due to improper pose handling. Pose-free methods that optimize 3D and pose separately risk inaccurate novel views resulting from improper pose estimations.

\noindent\textbf{Relative pose estimation}
The results of relative pose estimated are presented in Table \ref{tab:pose_performance}. Aligning with \cite{noposplat}, we assess the performance in pose estimation by training two models in different datasets. One model is trained using the RE10K dataset, while the other is trained using a combination of RE10K and DL3DV datasets, denoted as `Ours' and `Ours*' in the table, respectively. We evaluate the performance on both the RE10K test set and the ACID test set to demonstrate the efficacy of our models on both in-domain and cross-domain datasets. The results indicate that as the training set grows, the accuracy of pose estimation improves. The comparison on another evaluation matrix with rotation and translation error is shown in \cref{tab:pose_error}.

\begin{table}[!htp]
\centering
\caption{Pose estimation performance in AUC with various thresholds on RE10k, ACID. Methods with `*' meaning the model is trained on RE10k+DL3DV dataset, otherwise only on RE10K. `-A' means with evaluation-time pose alignment.}

\label{tab:pose_performance}

\begin{tabular}{lcccccc}
\toprule
 & \multicolumn{3}{c}{RE10k} & \multicolumn{3}{c}{ACID} \\
\cmidrule(lr){2-4} \cmidrule(lr){5-7}
Method & \ang{5} $\uparrow$ & \ang{10} $\uparrow$ & \ang{20} $\uparrow$ & \ang{5} $\uparrow$ & \ang{10} $\uparrow$ & \ang{20} $\uparrow$ \\
\midrule
CoPoNeRF & 0.161 & 0.362 & 0.575 & 0.078 & 0.216 & 0.398 \\
DUSt3R & 0.301 & 0.495 & 0.657 & 0.166 & 0.304 & 0.437 \\
MASt3R & 0.372 & 0.561 & 0.709 & 0.234 & 0.396 & 0.541 \\
RoMa & 0.546 & 0.698 & 0.797 & 0.463 & 0.588 & 0.689 \\
NoPoSplat & 0.546 & 0.719 & 0.843 & 0.366 & 0.519 & 0.654 \\
NoPoSplat* & 0.623 & 0.775 & 0.867 & 0.440 & 0.578 & 0.693 \\
NoPoSplat-A  & 0.672 & 0.792 & 0.869 & 0.454 & 0.591 & 0.709 \\
NoPoSplat-A*  & 0.691 & 0.806 & 0.877 & 0.486 & 0.617 & 0.728 \\
Ours & 0.605 & 0.749 & 0.885 & 0.405 & 0.556 & 0.661 \\
Ours* & 0.650 & 0.788 & 0.887 & 0.451 & 0.587 & 0.702 \\
Ours-A & 0.710 & 0.805 & 0.887 & 0.487 & 0.610 & 0.715 \\
Ours-A* & 0.683 & 0.810 & 0.886 & 0.499 & 0.627 & 0.735 \\
\bottomrule
\end{tabular}

\end{table}

\noindent\textbf{Geometry reconstruction}
To conduct a comprehensive comparison between our model and previous approaches based on the reconstructed 3D Gaussians rather than solely novel views, we present the rendered geometry depicted using 3D Gaussians in \cref{fig: geo}. The results of geometry are visualized by the code provided by \cite{pixelsplat, mvsplat}.
As illustrated in the figure, by jointly optimizing 3D Gaussians and input view camera poses, our methods achieve better alignment between input views, resulting in reduced artifacts within the reconstructed scenes. Furthermore, our model exhibits enhanced performance in reconstructing missing elements that were challenging for previous methods.


\noindent\textbf{Model Efficiency}
To show the model efficiency of our model, we test the inference time on a 3090 GPU with batch size 1. As shown in \cref{tab: efficiency}, our method can predict 3D Gaussians together with the camera poses from two $256 \times 256$ input images in 0.131 seconds (7.63 fps), outperforms previous methods.

\noindent\textbf{In-the-wild input images}
To ensure that our model can generalize beyond the training dataset, we evaluated it on in-the-wild images, as illustrated in \cref{fig:in_the_wild}. Our model demonstrates ability in both indoor and outdoor scenes.

\begin{figure}
    \centering
    \includegraphics[width=0.99\textwidth]{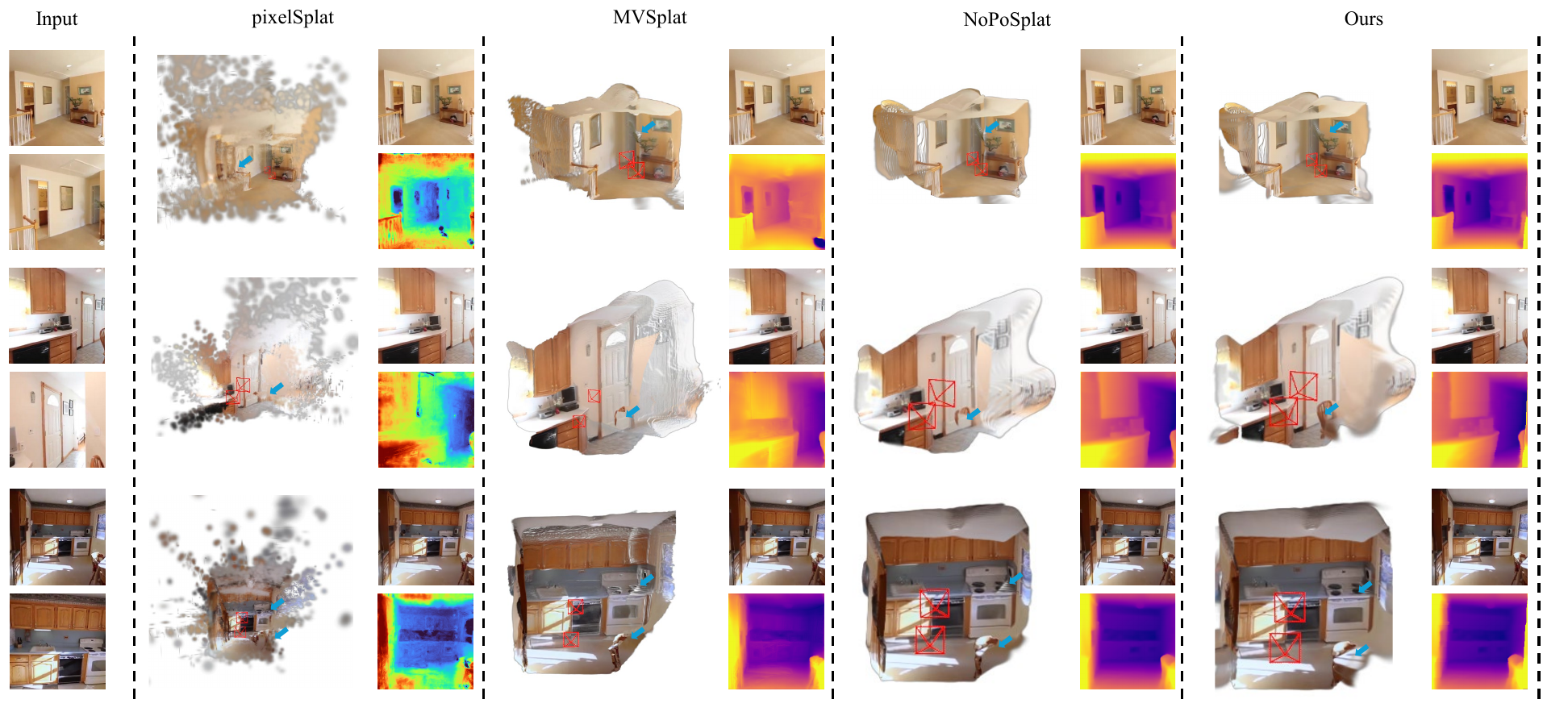}
    
    \caption{Geometry visualization with renderd novel view and depth. Our model demonstrates enhanced performance in addressing specific artifacts in the first row. Furthermore, our model excels in reconstructing missing elements that were challenging for prior approaches, (like chairs and tables situated in the corner), indicated by the \textcolor{blue}{blue} arrows. }
    \label{fig: geo}
    
\end{figure}

\begin{table}[!htp]
\centering
\caption{Inference time comparison. `3D time' denotes the duration for reconstructing 3D scenes, `pose time' signifies the time taken for pose estimation, `forward time' represents the total duration for inferring both 3D scenes and poses, while `rendering time' indicates the time for NVS from 3D for a single image. For the network without the parts or has no separate parts, we use `NA'. without All measurements are in seconds.}

\label{tab: efficiency}

\begin{tabular}{lcccc}
\toprule
 Method & 3D time & Pose time & Forward time & Rendering time \\
\midrule
pixelSplat & 0.825 & NA & 0.825 & \textbf{0.003} \\
MVSplat & 0.0.301 & NA & 0.301 & \textbf{0.003} \\
DBARF & NA & NA & 1.132 & 1.268 \\
DUSt3R & 0.062 & 0.211 & 0.273 & \textbf{0.003}  \\
MASt3R & 0.079 & 0.211 & 0.290 & \textbf{0.003} \\
NoPoSplat & 0.108 & 0.211 & 0.319 & \textbf{0.003} \\
Our model & NA & NA & \textbf{0.131} & \textbf{0.003} \\
\bottomrule
\end{tabular}

\end{table}

\subsection{Ablation studies}
\noindent\textbf{Ground-truth pose input}
In contrast to earlier pose-free methods lacking pose incorporation in their networks, our model optimizes 3D and pose jointly. Consequently, we can use the ground-truth poses as input to further enhance our model's performance give they are given. By utilizing ground-truth pose data, our model demonstrates even better outcomes, as evidenced in \cref{tab:ablation}.


\noindent\textbf{No {\camdfa}}
To demonstrate the significance of our custom {\camdfa} decoder, we conducted an ablation study by omitting this component. In other words, we attempted to regress 3D Gaussians directly from the ViT encoder. The absence of the {\camdfa} resulted in a decrease in performance as shown in \cref{tab:ablation}.

\noindent\textbf{Regress camera pose as 6D pose}
As discussed in \cref{sec: rays2cam}, directly regressing 6D pose of the camera from queries with numerous parameters, leads to suboptimal results. To support this, we conducted an ablation study on directly regressing the 6D pose (rotation and translation) as shown in \cref{tab:ablation}, with the results supporting our assertion.

\noindent\textbf{Rays on pixel}
In our model, we define rays from the camera center to 2D reference points, which differs from the approach in \cite{rays}, where rays are defined from the camera center to each pixel or patch on the image. To demonstrate the effectiveness of our approach, we conduct an ablation study on the rays defined in \cite{rays}, referred to as `Rays on pixel' in \cref{tab:ablation}. As illustrated in the table, the design of our {\rayref} contributes to improved model performance.

\begin{figure}[!htp]
    \centering
    \includegraphics[width=0.99\textwidth]{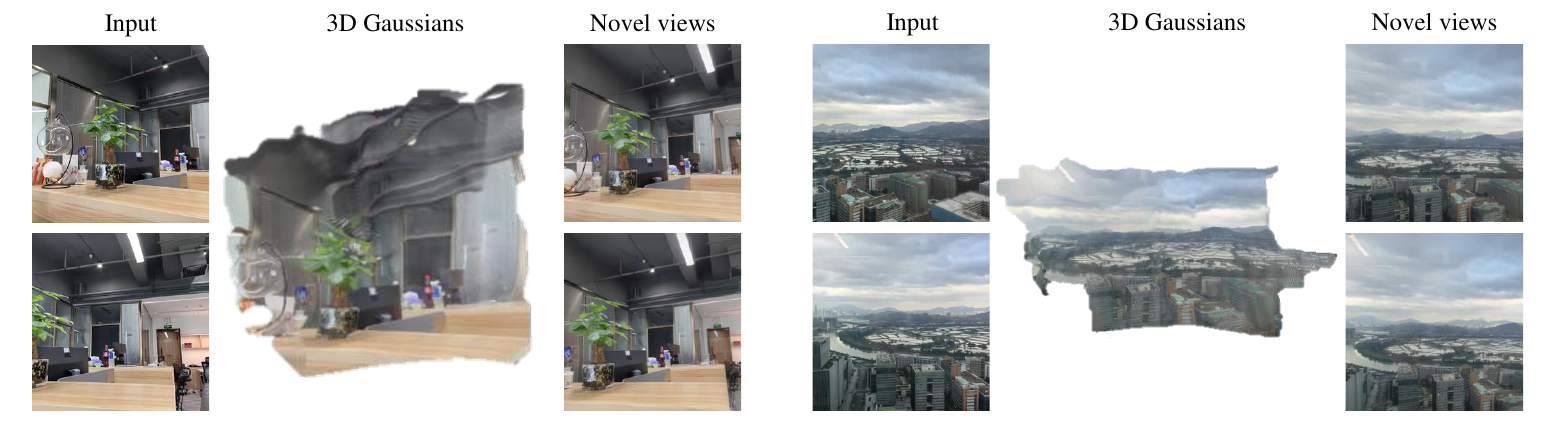}
    
    \caption{In-the-wild input images. Our method is applied using real-world photographs taken with mobile phones, encompassing both indoor and outdoor scenarios.}
    \label{fig:in_the_wild}
    
\end{figure}
\noindent\textbf{Input more views}
Our model is flexible and can receive an arbitrary number of input views rather than only two views. The inclusion of three input views yields superior results compared to two views, as demonstrated in \cref{tab:ablation}. Additionally, we present the PSNR results for a larger number of input views, extending up to 10 views, in the model details section (\cref{app: more_views} \cref{fig: more_views}).

\begin{table}[!htp]
\centering
\caption{Ablation Studies on RE10K dataset.}

\label{tab:ablation}

\begin{tabular}{lcccc}
\toprule
 Method & PSNR $\uparrow$ & SSIM $\uparrow$ & LPIPS $\downarrow$  \\
\midrule
Full model & 25.656 & 0.864 & 0.149 \\
\midrule
GT pose input & 26.217 & 0.870 & 0.138 \\
No {\camdfa} & 23.916 & 0.808 & 0.172\\
6D pose & 23.929 & 0.814 & 0.163 \\
Rays on pixel & 24.323 & 0.822 & 0.158 \\
Input 3 views & 26.706 & 0.884 & 0.124 \\
\bottomrule
\end{tabular}

\end{table}

\noindent\textbf{Number of Gaussians}
We evaluate our model using the same number of 3D Gaussians as previous methods, equivalent to the total pixels in two input views (131,072). However,our model offers the flexibility to treat the number of 3D Gaussians as a hyperparameter. In our ablation study on the number of 3D Gaussians (\cref{app: num_gaussian} \cref{fig: num_gaussian}), we observe that performance improves with an increased number of 3D Gaussians utilized.

\section{Conclusion and Limitation} \label{sec: conclusion}
In conclusion, we introduces a novel approach \textbf{{\methodname}} to optimizing 3D Gaussians and camera poses collaboratively within a unified network, enhancing the relationship between them for improved scene reconstruction and novel view synthesis. Inspired by Deformable DETR, our model represents define separate queries for 3D Gaussians and camera parameters, leveraging image features for efficient updates within \textbf{{\camdfa}} block. This process involves projecting 3D queries onto image features using camera poses as 2D reference points and utilizing surrounding image features to update both 3D and camera queries. To tackle the challenge in regressing rotations and translations from numerous parameters, we adopt a ray-based design ({\rayref}) that representing camera poses as rays from the camera center to the 2D reference points on the image feature.
Our methodology surpasses previous approaches on benchmark datasets like RealEstate10K and ACID, demonstrating its effectiveness and potential for advancing scene reconstruction and novel view synthesis. 

\noindent\textbf{Limitations} A limitation persists in that the training datasets seldom contain videos that capture the entire 360-degree scenes. Consequently, our model lack the ability on reconstruct the \ang{360} scenes from one forward.
\clearpage

{
    \small
    \bibliographystyle{ieeenat_fullname}
    \bibliography{main}
}
\clearpage

\appendix
\section{Appendix}

\subsection{Model details}
\subsubsection{Comparison to previous methods}  \label{app: method_comparison}
\begin{table}[!htp]
\centering
\caption{Methods comparison: `Pose-Free' means that the model is pose-free, `Jointly Opt.' means that the model optimize camera and 3D jointly, `Cano. Space' means that the model define 3D points in the canonical space for all views, `Intri.-Free' means that the model does not require intrinsic input.}
\label{tab:method_comparison}

\begin{tabular}{lcccc}
\toprule
 Method & Pose-Free & Jointly Opt. & Cano. Space & Intri.-Free \\
\midrule
PixelNeRF & \XSolidBrush & \XSolidBrush & \XSolidBrush & \XSolidBrush\\
PixelSplat & \XSolidBrush & \XSolidBrush & \XSolidBrush & \XSolidBrush \\
MVSplat & \XSolidBrush & \XSolidBrush & \XSolidBrush & \XSolidBrush\\
DUSt3R & \Checkmark & \XSolidBrush & \Checkmark & \XSolidBrush\\
MASt3R & \Checkmark & \XSolidBrush & \Checkmark & \XSolidBrush\\
CoPoNeRF & \Checkmark & \Checkmark & \Checkmark & \XSolidBrush \\
SelfSplat & \Checkmark & \XSolidBrush & \XSolidBrush & \XSolidBrush\\
NoPoSplat & \Checkmark & \XSolidBrush & \Checkmark & \XSolidBrush\\
Ours  & \Checkmark & \Checkmark & \Checkmark & \Checkmark  \\
\bottomrule
\end{tabular}
\end{table}
We offer a comparison with various methods in \cref{tab:method_comparison}. Our model is a pose-free approach that jointly optimizes pose and scene, operating within canonical space, and is free from intrinsic constraints.

\subsection{RefRay}  \label{app: proof}
Typically, camera for the $i$-th view is parameterized by extrinsics $\boldsymbol{\pi}_i$ (consisting of rotations $\boldsymbol{\pi}_{i_\textbf{R}} \in SO(3)$ and translations $\boldsymbol{\pi}_{i_\mathbf{t}} \in \mathbb{R}^3$) and intrinsics $\textbf{K}_i \in \mathbb{R}^{3 \times 3}$. However, directly regressing $\textbf{K}_i$ and $\boldsymbol{\pi}_i$ from the camera queries ($\textbf{Q}_{\text{cam}_i} \in \mathbb{R}^{N \times D}$ for the $i$-th view), which entails regressing only 16 parameters (9 for the rotation matrix, 3 for translation, and 4 for intrinsics including focal lengths and principal points) from the $N \times D$ parameters through an MLP, may be suboptimal for neural learning~\cite{rays}. 
Therefore, inspired by \cite{rays}, we represent the camera parameters by $N$ rays, transferring starting from the camera center to the $N$ 2D reference points on the image plane as shown in \cref{fig: dfa}. 
More specifically, our approach over-parameterize camera in the $i$-th view by a set of rays $\boldsymbol{\cal{R}}_i = {\textbf{r}_{i1}, \ldots, \textbf{r}_{iN}}$, where each ray $\textbf{r}_{ij} \in \mathbb{R}^6$ is represented by Pl\"{u}cker coordinates \citep{plucker}: $\textbf{r}_{ij} \in \mathbb{R}^6$. 
Different from previous methods \citep{rays}, we define the rays sparsly on the 2D refernce points instead of on each pixel to enhance the relationship between 3D Gaussians and camera parameters. After each layer, once the rays are predicted by \cref{eq: ray_update}, the camera parameters can be sovled by RQ-decomposition on a overdetermined system of equations from the Pl\"{u}cker coordinates of RefRay with details of  Pl\"{u}cker-camera sparse mapping as following. We give the mapping from cameras to rays in \cref{cam_free_theorem: c2r} and the mapping from rays to cameras in \cref{cam_free_theorem: r2c}

\begin{theorem} \label{cam_free_theorem: c2r}
\textbf{Cameras to rays mapping.}
Given camera center $\mathbf{c} \in \mathbb{R}^{3}$, rotation matrix $\mathbf{R} \in \mathbb{R}^{3 \times 3}$, translation $\mathbf{t} \in \mathbb{R}^3$, camera intrinsics $\mathbf{K} \in\mathbb{R}^{3 \times 3}$, and a point $\mathbf{x} \in  \mathbb{R}^3$ in 3D with coordinate. The Pl\"{u}cker coordinates $\mathbf{d}, \mathbf{m}$ for a ray passing through points$ \mathbf{x}$ and $\mathbf{c}$, which defined as $\mathbf{d} = \mathbf{x} - \mathbf{c}$, $\mathbf{m} = \mathbf{c} \times \mathbf{d}$ can be derived by 
\begin{equation}
    \mathbf{d} = \mathbf{R}^T\mathbf{K}^{-1} \mathbf{u}, \mathbf{m} = (-\mathbf{R}^T \mathbf{t}) \times \mathbf{d}
\end{equation}
\end{theorem}

\begin{proof}
Given $\mathbf{x}_W, \mathbf{x}_C$ denote the coordinates of point $\mathbf{x}$ in world and camera coordinates and camera coordinates, $\mathbf{c}_W,\mathbf{c}_C$ denote the coordinates of camera center $\mathbf{c}$ in world and camera coordinates and camera coordinates. Let $\mathbf{u}$ denotes the UV coordinates that derived from projecting $\mathbf{x}$ to the image by camera parameters $\mathbf{u} = \mathbf{K}(\mathbf{R}\mathbf{x}_W + \mathbf{t})$.

\begin{align}
    \mathbf{c}_C & = \mathbf{R}\mathbf{c}_W + \mathbf{t} \\
    0 & = \mathbf{R}\mathbf{c}_W + \mathbf{t} \\
    \Rightarrow \mathbf{c}_W & = -\mathbf{R}^T \mathbf{t}
\end{align}

\begin{align}
    \mathbf{x}_C & = \mathbf{R}\mathbf{x}_W + \mathbf{t} \\
    \mathbf{K}^{-1} \mathbf{u} & = \mathbf{x}_C = \mathbf{R}\mathbf{x}_W + \mathbf{t} \\
    \mathbf{x}_W & = \mathbf{R}^T(\mathbf{K}^{-1} \mathbf{u} - \mathbf{t}) \\
    & = \mathbf{R}^T\mathbf{K}^{-1} \mathbf{u} - \mathbf{c}_C
\end{align}
Then use the definition of $\mathbf{d}, \mathbf{m}$,
\begin{align}
    \mathbf{d} & = \mathbf{x}_W - \mathbf{c}_W \\
    & = \mathbf{R}^T\mathbf{K}^{-1} \mathbf{u}
\end{align}
\begin{align}
    \mathbf{m} & = \mathbf{c} \times \mathbf{d} \\
    & = (-\mathbf{R}^T \mathbf{t}) \times \mathbf{d}
\end{align}

\end{proof}

\begin{theorem} \label{cam_free_theorem: r2c}
\textbf{Rays to cameras mapping.}
Given a set of rays $\boldsymbol{\cal{R}} = {\mathbf{r}_{1}, \ldots, \mathbf{r}_{N}}$, which denote the $N$ rays representing a camera with camera center $\mathbf{c} \in \mathbb{R}^{3}$, rotation matrix $\mathbf{R} \in \mathbb{R}^{3 \times 3}$, translation $\mathbf{t} \in \mathbb{R}^3$, camera intrinsics $\mathbf{K} \in\mathbb{R}^{3 \times 3}$, and projected points $\mathbf{U}$ on the image plane from 3D points $\mathbf{X}$, the camera parameters can be solved by,
\begin{equation}
    \mathbf{c} = \arg \min_{\mathbf{p} \in \mathbb{R}^3} \sum_{\langle \mathbf{d}, \mathbf{m} \rangle \in \mathcal{R}} | \mathbf{p} \times \mathbf{d} - \mathbf{m} |^2
\end{equation}
\begin{equation} \label{eq: cam_rot}
    \mathbf{M} = \arg\min_{|\mathbf{H}|=1} \sum_{i=1}^m |\mathbf{H}\mathbf{d}_j \times \mathbf{u}_j|
\end{equation}
where $\mathbf{r} = \langle\mathbf{d}, \mathbf{m} \rangle \in \boldsymbol{\cal{R}}$ denotes a ray passing through points $\mathbf{x}$ and camera center $\mathbf{c}$, $\mathbf{d} = \mathbf{x} - \mathbf{c}$, $\mathbf{m} = \mathbf{c} \times \mathbf{d}$, $\mathbf{M} = \mathbf{K}\mathbf{R}$.
\end{theorem}

\begin{proof}
    All rays should passing through the same point $\mathbf{c}$, which is the camera center. In order to solve the camera center, we want 
    \begin{equation}
        \mathbf{m} = \mathbf{c} \times \mathbf{d}
    \end{equation}
    for all rays $\mathbf{r} = \langle\mathbf{d}, \mathbf{m} \rangle \in \boldsymbol{\cal{R}}$. When the number of rays is greater than the the number of unknown parameters for camera center, this gives a overparameterized equation system.
    \begin{equation}
        \mathbf{c} = \arg \min_{\mathbf{p} \in \mathbb{R}^3} \sum_{\langle \mathbf{d}, \mathbf{m} \rangle \in \mathcal{R}} | \mathbf{p} \times \mathbf{d} - \mathbf{m} |^2
    \end{equation}
    As for the rotation matrix $\mathbf{K}$ and intrinsics matrix $\mathbf{K}$, denoting $\mathbf{M} = \mathbf{K}\mathbf{R}$,
    \begin{align}
        \mathbf{d} & = \mathbf{R}^T\mathbf{K}^{-1} \mathbf{u} \\
         &\Rightarrow \mathbf{K}\mathbf{R}\mathbf{d} = \mathbf{u} \\
         &\Rightarrow = \mathbf{M}\mathbf{d} = \mathbf{u}.
    \end{align}
    Ignoring the translation here, we require the direction of $\mathbf{M}\mathbf{d}$ and $\mathbf{u}$ being the same, i.e., 
    \begin{equation}
        \mathbf{M}\mathbf{d} \times \mathbf{u} = 0. 
    \end{equation}
    Similar to what we do for camera center, $\mathbf{M}$ can be solved by
    \begin{equation}
    \mathbf{M} = \arg\min_{|\mathbf{H}|=1} \sum_{i=1}^m |\mathbf{H}\mathbf{d}_j \times \mathbf{u}_j|
    \end{equation}
\end{proof}

\subsection{Details for spatially efficient self-attention} \label{app: self_attn}
Directly do self-attention on the 3D queries may cause large memory costs, particularly when $N$ is extensive, we implement a spatially efficient self-attention strategy inspired by \citep{unig}.
This method involves sampling a subset of queries $\textbf{Q}_{down}$ from  $\textbf{Q}_{3D}$ using the Fast Point Sampling (FPS) \citep{pointnet} algorithm on the locations of the queries, which are defined as the centers of 3D Gaussians. Such design focuses on 3D Gaussians positioned farthest apart and likely to encompass the entire scene. Keys and values are then derived from $\textbf{Q}_{down}$ while queries are extracted from $\textbf{Q}_{3D}$ by linear projection during self-attention computation. This approach minimizes memory consumption without compromising performance. 

\subsection{Details for 3D queries initialization} \label{app: init}
3D queries are initialized by the flattening image features with shape $V \times H \times W$, where $V$, $H$, and $W$ represent the number of input views, height, and width of each view, respectively.
Subsequently, the 3D queries are passed through $\mathrm{Head}_{3D}$ to initialize 3D Gaussians.

In cases where the resolution or the number of input views are large, resulting in large $V \times H \times W$, memory conservation becomes crucial. To address this, we introduce a hyperparameter $N_{3D}$ to denote the number of 3D queries. To select $N_{3D}$ queries from the pool of $V \times H \times W$ queries, we employ FPS (Farthest Point Sampling) \citep{pointnet} on the centers of the initial 3D Gaussians. Given the one-to-one correspondence between the 3D Gaussians and 3D queries, the 3D queries are subsequently downsampled based on the chosen 3D Gaussians.

Note that the number of camera queries can be much smaller than 3D queries due to the sparser nature of camera parameters requirements, we subsample a portion of the 2D reference points using FPS \citep{pointnet} based on the centroids of 3D Gaussians to obtain the reference points for camera parameters queries. For simplicity, we ignore the downsampling in the main paper. In the experiment, we use 256 queries for camera parameters in each view while 131,072 queries for 3D Gaussians.

\subsection{Details for deformable attention} \label{app: model_details}
\begin{figure}
    \centering
    \includegraphics[width=0.36\textwidth]{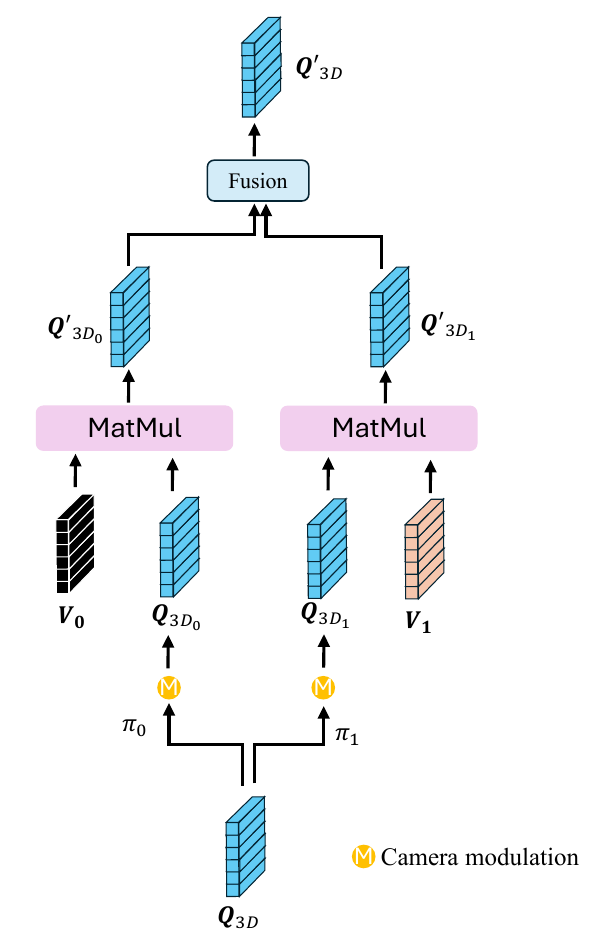}
    \caption{Query updation for CAMVDFA. The updating of queries for 3D Gaussians follows a similar process to that of camera queries. In this context, 'MatMul' signifies matrix multiplication. $\mathbf{Q}{3D_i}$ stands for the 3D queries for the $i$-th input view, while $\mathbf{Q}'{3D_i}$ indicates the updated 3D queries for the same view. Camera modulation has details in \cref{sec: dfa}. Further insights into the fusion block, which combines the view-wise queries.}
    \label{fig: dfa_query_update}
\end{figure}

After defining the reference points $\mathbf{P}_i$, image features, $\textbf{F}_i$, and given the 3D queries $\textbf{Q}_{3D_i}$ for the $i$-th view, we do DFA on them to refine the queries.
To sample the image features surrounding the reference points, trainable sampling offsets $\Delta \mathbf{s}=\mathrm{MLP}(\mathbf{q})$ are calculated, where $\mathrm{MLP}$ represents the multilayer perceptron layer comprising two linear layers and two activation functions. After that, the image features on the sampling points $\mathbf{s} = \mathbf{P}_i + \Delta \mathbf{s}$ are sampled by grid sampling algorithm with bilinear interpolation, serving as the values $\mathbf{v}$ for cross-attention. Then the attention score $a_{3D_i} = \mathrm{MLP}(\textbf{Q}_{3D_i})$ is computed for the sampled image features at $\mathbf{s}$. Finally, $\textbf{Q}_{3D_i}$ is refined by dot product between the attention score and values. Consistent with \citep{dino, DETR}, we derive attention scores directly from queries, omitting keys to simplify computations. The updating process of queries is shown in \cref{fig: dfa_query_update}.
\subsection{More results}
\paragraph{Pose estimation performance}

\begin{table}[!htp]
\centering
\caption{Pose estimation performance in rotation error (e\_rot) and translation error (e\_trans) on RE10k, ACID. Methods with `*' meaning the model is trained on RE10k+DL3DV dataset, otherwise only on RE10K. `-A' means with evaluation-time pose alignment.}
\label{tab:pose_error}

\begin{tabular}{lcccc}
\toprule
 & \multicolumn{2}{c}{RE10k} & \multicolumn{2}{c}{ACID}\\
\cmidrule(lr){2-3} \cmidrule(lr){4-5}
Method & e\_rot $\downarrow$ & e\_trans $\downarrow$ & e\_rot $\downarrow$ & e\_trans $\downarrow$ \\
\midrule
NoPoSplat & 4.258 & 2.879 & 11.269 & 3.296 \\
NoPoSplat* & 3.399 & 1.806 & 9.835 & 3.827 \\
NoPoSplat-A & 3.860 & 2.953 & 10.507 & 3.423 \\
NoPoSplat-A* & 2.948 & 1.912 & 10.448 & 3.488 \\
Ours & 3.704 & 2.295 & 11.087 & 3.192 \\
Ours* & 3.164 & 1.746 & 9.797 & 2.988 \\
Ours-A & 3.295 & 2.704 & 10.191 & 3.053 \\
Ours-A* & 2.746 & 1.838 & 9.815 & 3.234 \\
\bottomrule
\end{tabular}

\end{table}

In addition to the pose estimation evaluated by AUC, we also provide the pose estimation performance in rotation error (e\_rot) and translation error (e\_trans) on RE10k, ACID in \cref{tab:pose_error}.

\paragraph{Performance on the evaluation set of pixelSplat}
\begin{table}[!htp]
\centering
\caption{Performance comparison on the evaluation set of pixelSplat for re10k dataset.}
\label{tab:pixelsplat_index}

\begin{tabular}{lcccc}
\toprule
 Method & PSNR $\uparrow$ & SSIM $\uparrow$ & LPIPS  \\
\midrule
pixelNeRF & 20.43 & 0.589 & 0.55 \\
GPNR & 24.11 & 0.793 & 0.255 \\
AttnRend & 24.78 & 0.82 & 0.213\\
pixelSplat & 26.09 & 0.863 & 0.136 \\
MVSplat & 26.39 & 0.869 & 0.128 \\
NoPoSplat-A & 26.786 & 0.878 & 0.124 \\
Ours & 25.837 & 0.852 & 0.129 \\
Ours-A & 27.214 & 0.891 & 0.118\\
\bottomrule
\end{tabular}
\end{table}

In order to conduct a fair comparison between our methods and the previous MVSplat benchmark, we evaluated our model using the indices provided in pixelSplat and MVSplat \citep{pixelsplat, mvsplat}. As depicted in \cref{tab:pixelsplat_index}, our model continues to outperform previous methodologies on the pixelSplat index.

\paragraph{More number of input views} \label{app: more_views}
\begin{figure}
    \centering
    \includegraphics[width=0.45\textwidth]{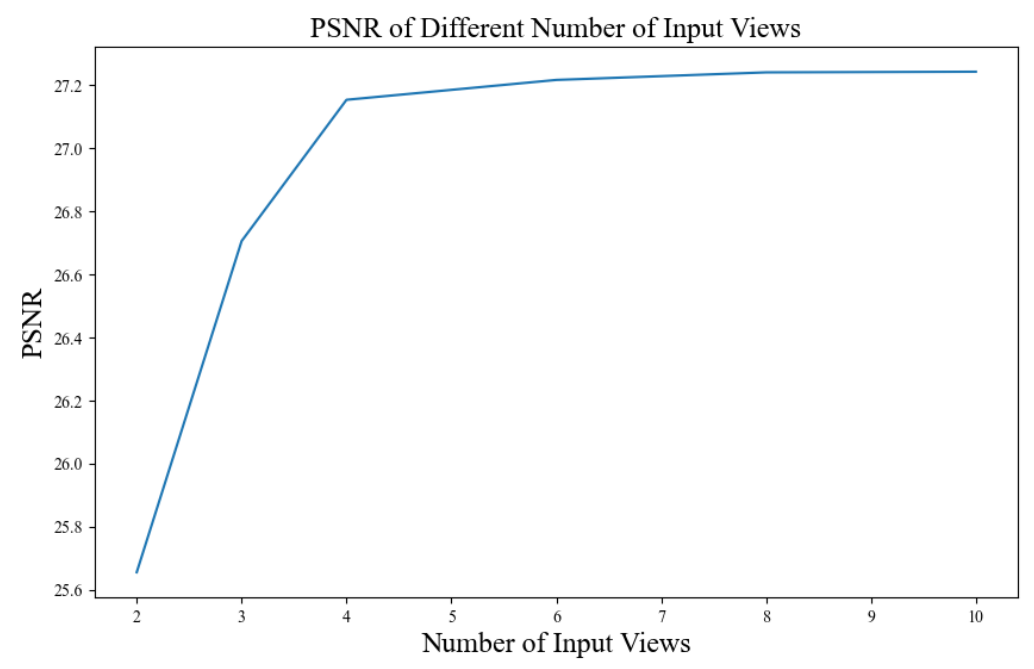}
    \caption{PSNR v.s. number of input views.}
    \label{fig: more_views}
\end{figure}
Our model is flexible and can receive an arbitrary number of input views, surpassing the limitations of only considering two views. We present the PSNR results for a larger number of input views, extending up to 10 views, in \cref{fig: more_views}. Generally, more input views give better performance on scene reconstruction.

\paragraph{Number of 3D Gaussians} \label{app: num_gaussian}
\begin{figure}
    \centering
    \includegraphics[width=0.45\textwidth]{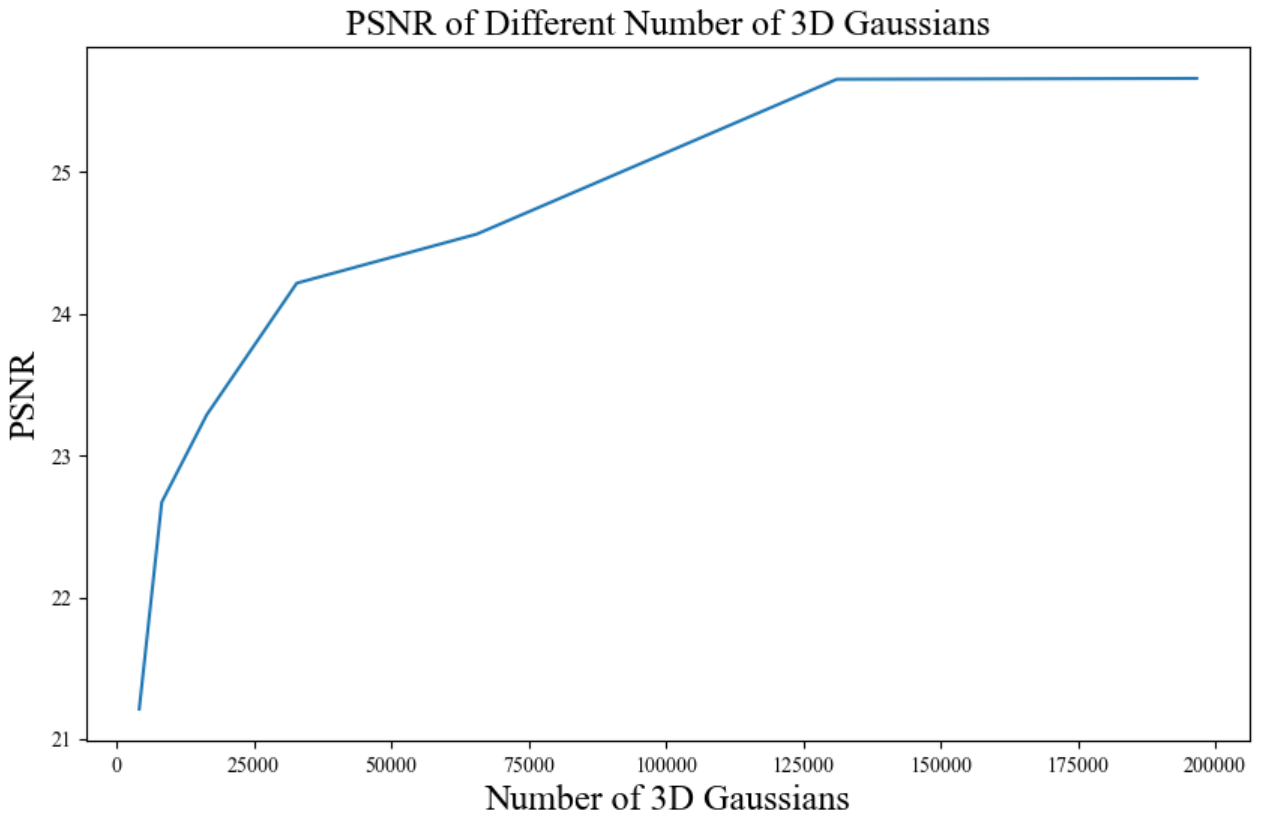}
    \caption{PSNR v.s. number of 3D Gaussians.}
    \label{fig: num_gaussian}
\end{figure}

We show the PSNR v.s. the number of 3D Gaussians in \cref{fig: num_gaussian}.

\paragraph{More visualization} \label{app: more_vis}

We provide more novel view synthesis results on RE10K and ACID with different groups of overlapping in this section \cref{fig: re10k_large}, \cref{fig: re10k_medium}, \cref{fig: re10k_small}, \cref{fig: acid_large}, \cref{fig: acid_medium}, \cref{fig: acid_small}. 

\begin{figure*}
    \centering
    \includegraphics[width=0.95\textwidth]{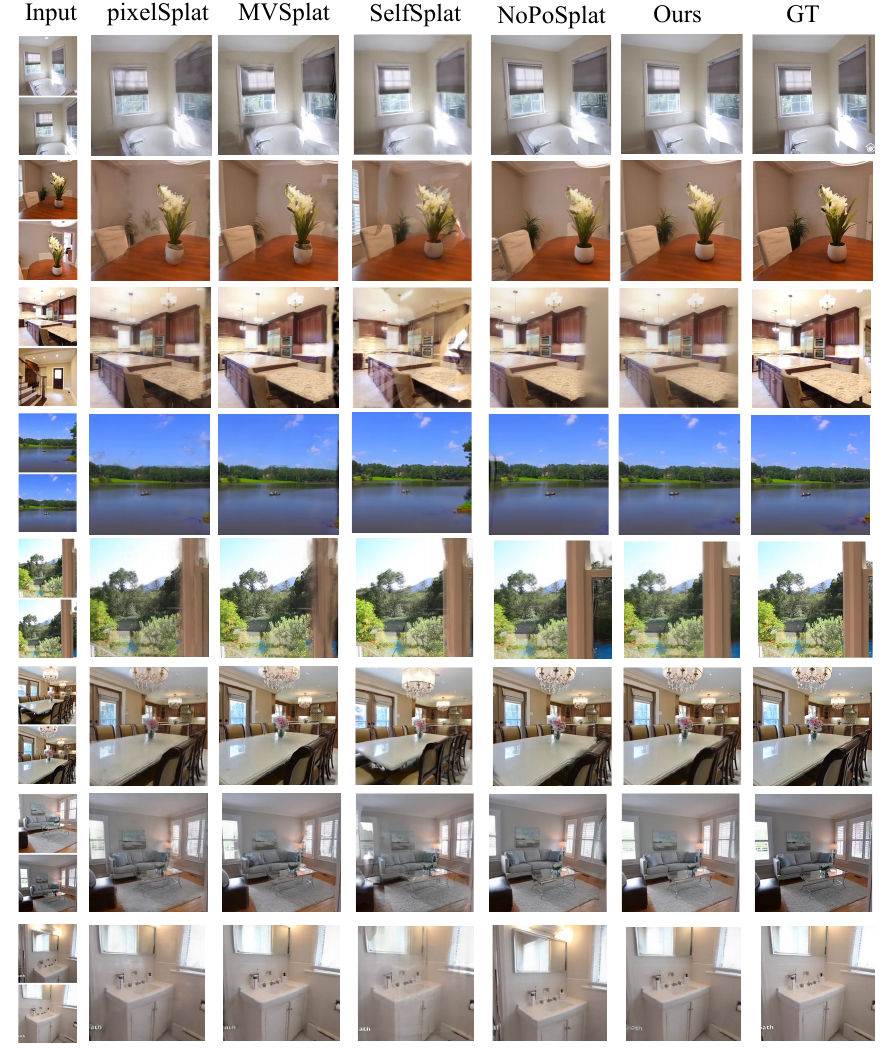}
    \caption{More comparisons of the RealEstate10K dataset with large overlap of input images.}
    \label{fig: re10k_large}
\end{figure*}

\begin{figure*}
    \centering
    \includegraphics[width=0.95\textwidth]{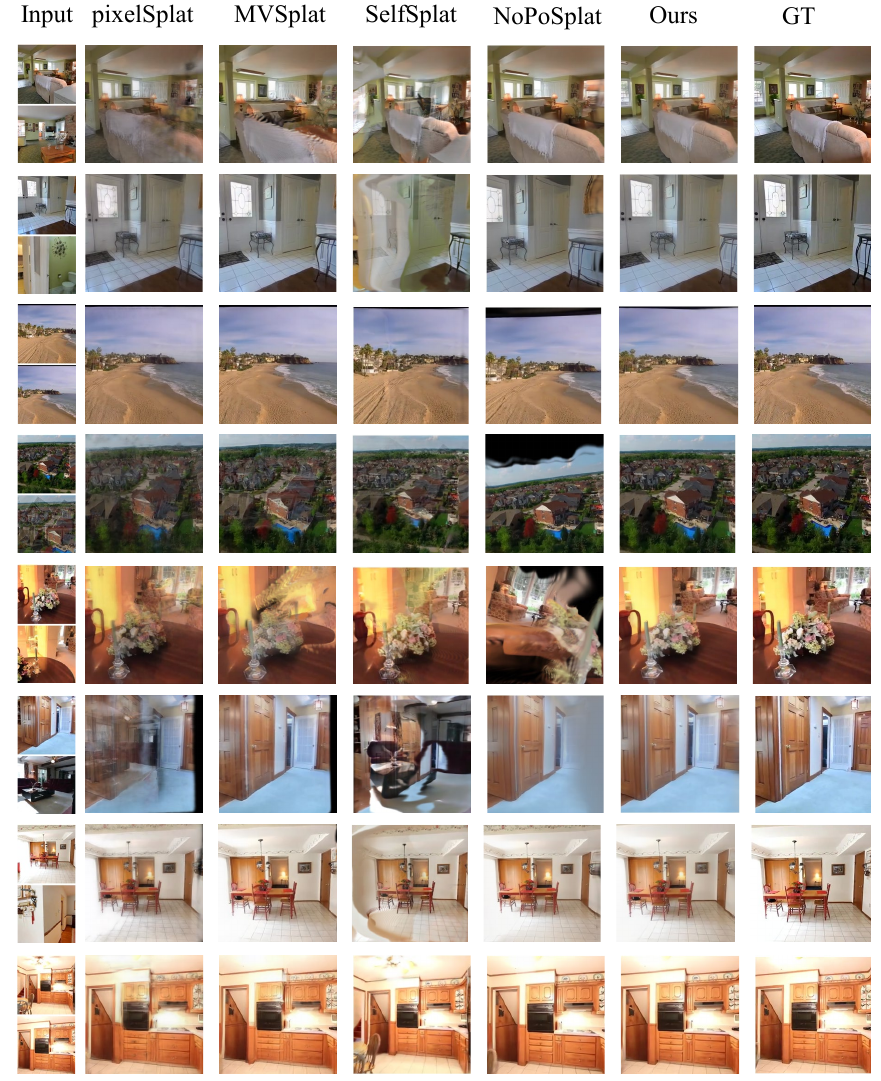}
    \caption{More comparisons of the RealEstate10K dataset with medium overlap of input images.}
    \label{fig: re10k_medium}
\end{figure*}

\begin{figure*}
    \centering
    \includegraphics[width=0.95\textwidth]{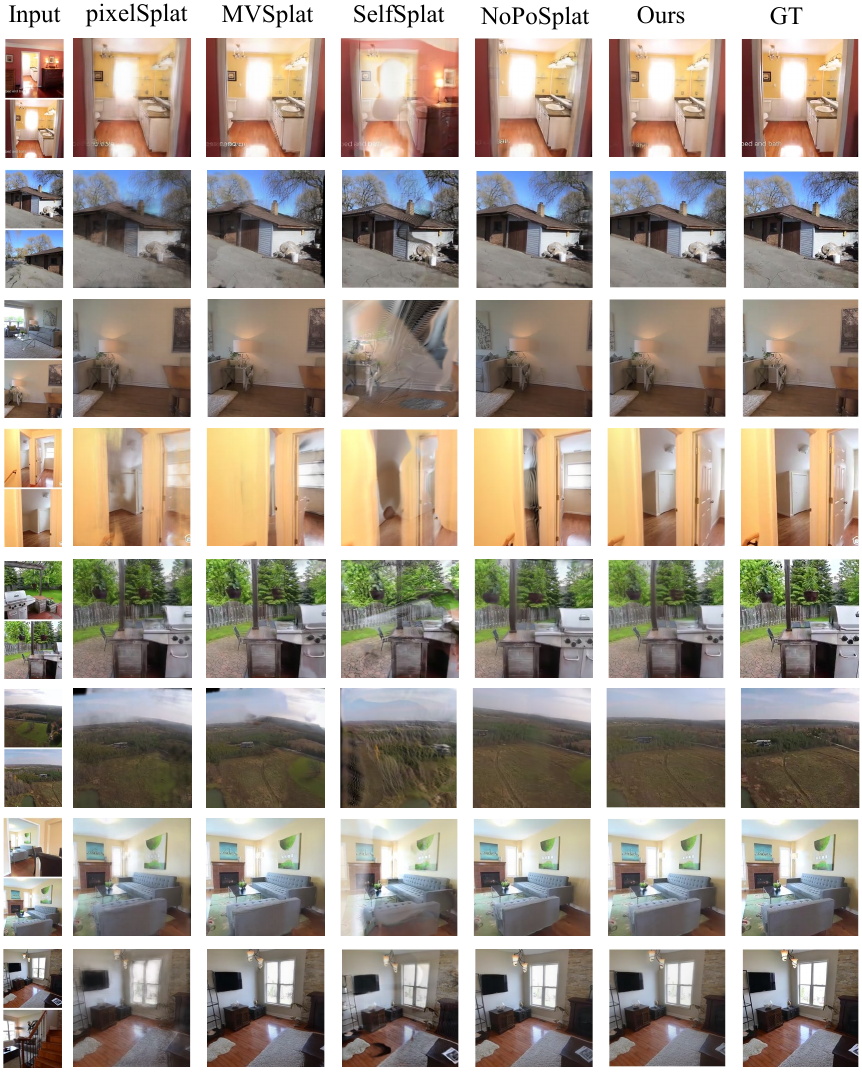}
    \caption{More comparisons of the RealEstate10K dataset with small overlap of input images.}
    \label{fig: re10k_small}
\end{figure*}

\begin{figure*}
    \centering
    \includegraphics[width=0.95\textwidth]{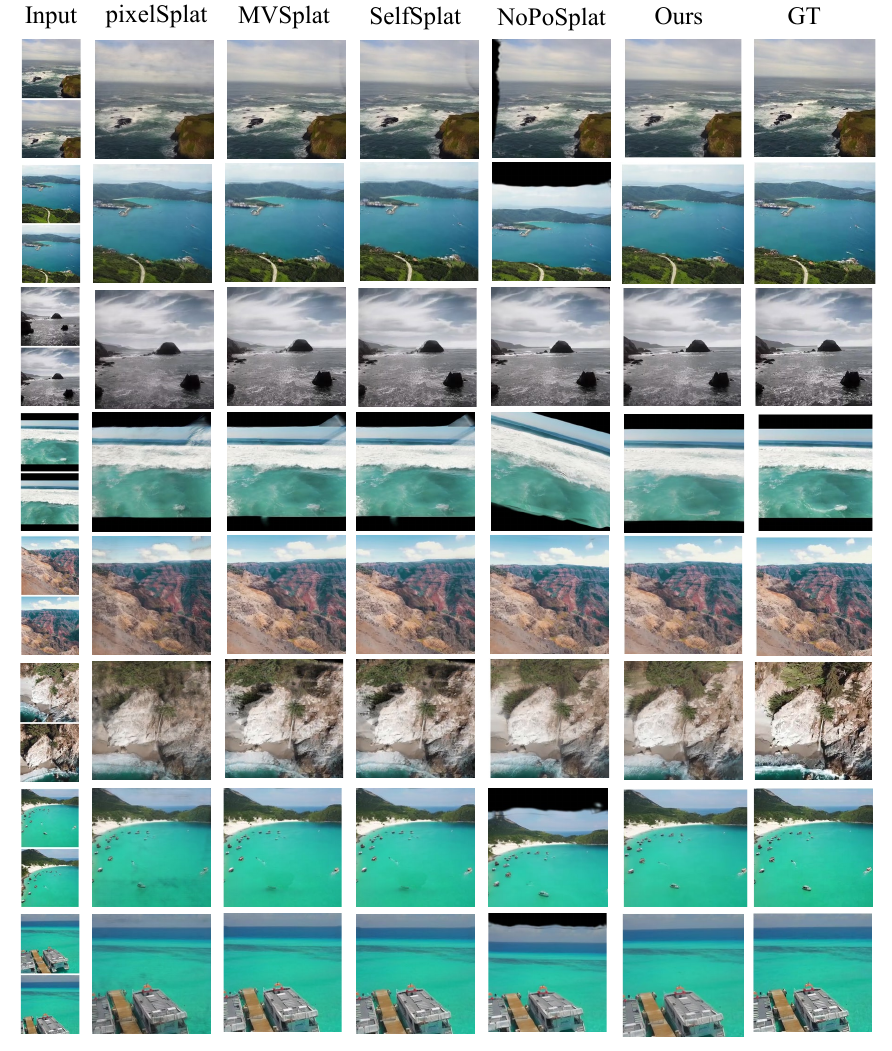}
    \caption{More comparisons of the ACID dataset with large overlap of input images.}
    \label{fig: acid_large}
\end{figure*}

\begin{figure*}
    \centering
    \includegraphics[width=0.95\textwidth]{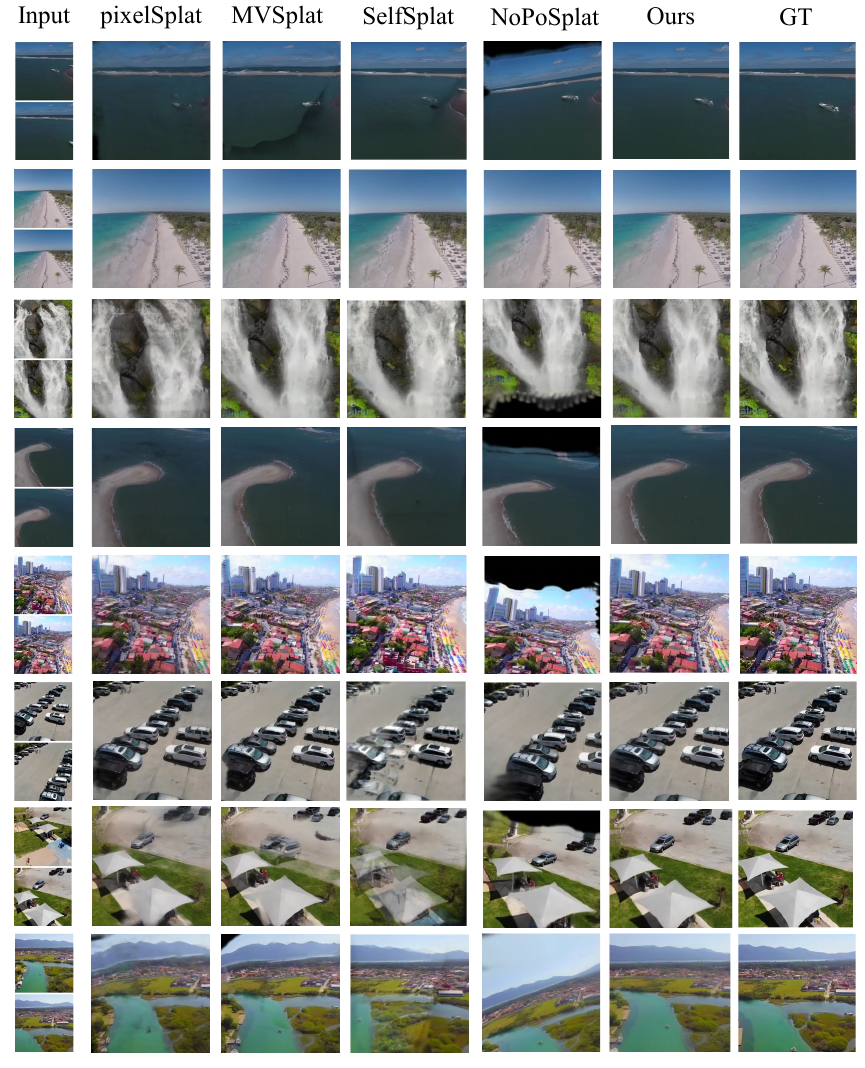}
    \caption{More comparisons of the ACID dataset with medium overlap of input images.}
    \label{fig: acid_medium}
\end{figure*}

\begin{figure*}
    \centering
    \includegraphics[width=0.95\textwidth]{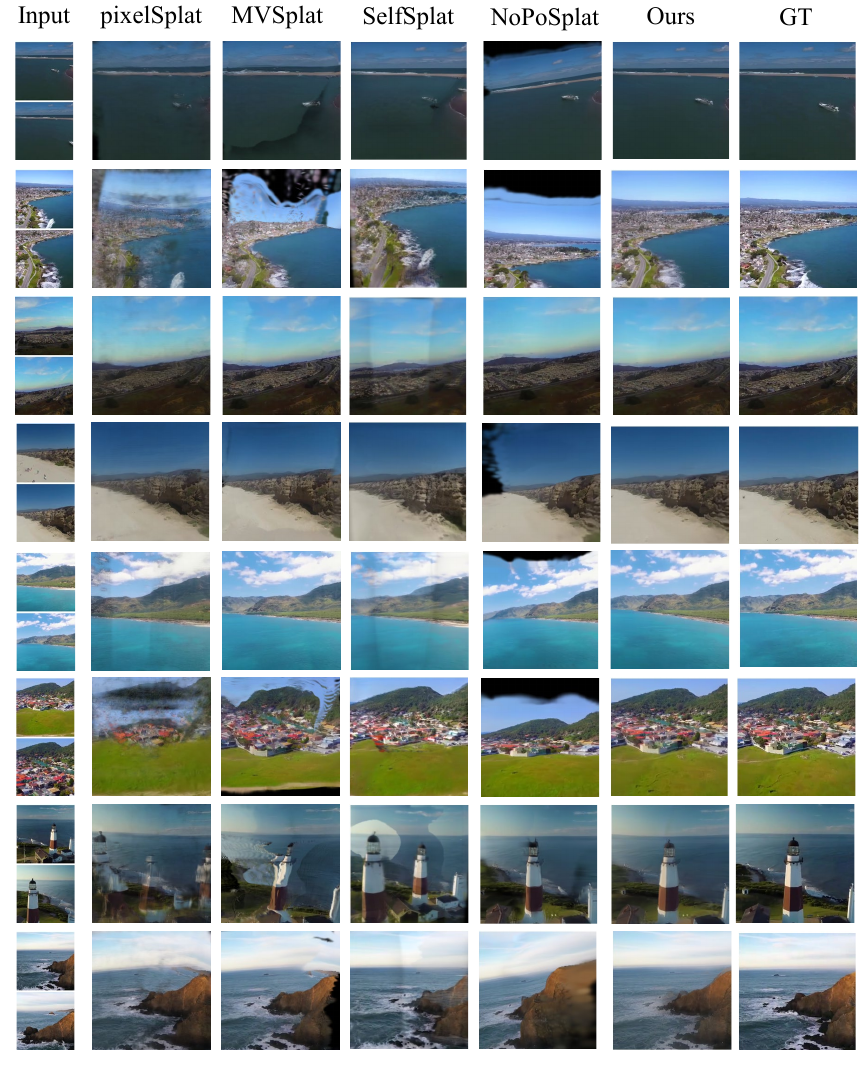}
    \caption{More comparisons of the RealEstate10K dataset with small overlap of input images.}
    \label{fig: acid_small}
\end{figure*}
\end{document}